%% file: acl_latex.tex
\pdfoutput=1

\RequirePackage{amsmath}
\documentclass[11pt]{article}

\usepackage[preprint]{acl}

\usepackage[amsthm]{newtx}
\usepackage{latexsym}

\usepackage[T1]{fontenc}

\usepackage[utf8]{inputenc}

\usepackage{microtype}

\usepackage{inconsolata}

\usepackage{graphicx}

\usepackage{booktabs}           %
\usepackage{bm}                 %
\usepackage{xspace}
\usepackage{amsthm}
\usepackage{multicol}
\usepackage{multirow}
\usepackage{mathtools}
\usepackage{nicefrac}
\usepackage{colortbl}
\usepackage{nicematrix}
\usepackage[titlenumbered,linesnumbered,ruled]{algorithm2e}
\usepackage{subcaption}
\usepackage{array}
\usepackage{thmtools, thm-restate}
\usepackage{anyfontsize}
\usepackage{enumitem}

\input{defs.tex}

\title{CERT-ED: Certifiably Robust Text Classification for Edit Distance}

\author{Zhuoqun Huang, Neil G.\ Marchant, Olga Ohrimenko, Benjamin I.\ P.\ Rubinstein \\
  School of Computing and Information Systems \\
  University of Melbourne \\
  Melbourne, Australia \\
  \texttt{\{zhuoqun, nmarchant, oohrimenko, brubinstein\}@unimelb.edu.au}}

\begin{document}
\maketitle

\begin{abstract}
  With the growing integration of AI in daily life,
  ensuring the robustness of systems to inference-time attacks is crucial.
  Among the approaches for certifying robustness to such adversarial examples,
  randomized smoothing has emerged as highly promising due to its nature as a wrapper around arbitrary black-box models.
  Previous work on randomized smoothing in natural language processing has primarily focused
  on specific subsets of edit distance operations, such as synonym substitution or word insertion,
  without exploring the certification of all edit operations.
  In this paper, we adapt \emph{Randomized Deletion} \citep{huang2023rsdel} and propose, 
  CERTified Edit Distance defense (\randel) for natural language classification. 
  Through comprehensive experiments, we demonstrate that \randel\ outperforms the existing Hamming distance
  method \ranmask\ \citep{zeng-etal-2023-certified} in 4 out of 5 datasets in terms
  of both accuracy and the cardinality of the certificate.
  By covering various threat models, including 5 direct and 5 transfer attacks,
  our method improves empirical robustness in 38 out of 50 settings.
\end{abstract}

\input{sections/introduction.tex}
\input{sections/preliminaries.tex}

\input{sections/method.tex}

\input{sections/evaluation.tex}

\input{sections/related.tex}

\input{sections/conclusion.tex}

\input{sections/ethics.tex}
\input{sections/limitation.tex}

\bibliography{custom}

\appendix

\input{appendices/proofs.tex}
\input{appendices/evaluation-parameters.tex}
\input{appendices/evaluation-certify.tex}
\input{appendices/evaluation-attack.tex}
\input{appendices/evaluation-efficiency.tex}

\input{appendices/textcrs.tex}

\end{document}

%% file: defs.tex
\renewcommand{\vec}[1]{\bm{#1}}

\newcommand{\scrE}{\ensuremath{\mathcal{E}}}

\newcommand{\scrX}{\ensuremath{\mathcal{X}}}
\newcommand{\scrY}{\ensuremath{\mathcal{Y}}}

\newcommand{\del}{\ensuremath{\mathsf{del}}}
\newcommand{\ins}{\ensuremath{\mathsf{ins}}}
\newcommand{\sub}{\ensuremath{\mathsf{sub}}}

\newcommand{\mask}{\ensuremath{\mathsf{mask}}}

\newcommand{\mdel}{\ensuremath{\epsilon}}

\newcommand{\medit}{\ensuremath{\vec{\epsilon}}}
\newcommand{\pert}[1]{\bar{#1}}

\newcommand{\radius}{r}
\newcommand{\smooth}[1]{#1}

\newcommand{\base}[1]{{#1}_{\mathrm{b}}}
\newcommand{\reals}{\mathbb{R}}

\newcommand{\ind}[1]{\mathbf{1}_{#1}}

\newcommand{\tokenizer}{\mathsf{t}}
\newcommand{\opset}{\mathsf{o}}

\newtheorem{theorem}{Theorem}
\newtheorem{lemma}[theorem]{Lemma}

\newtheorem{proposition}[theorem]{Proposition}
\theoremstyle{definition}
\newtheorem{definition}{Definition}
\theoremstyle{remark}

\DeclareMathOperator{\dist}{dist}
\DeclareMathOperator{\apply}{apply}

\DeclareMathOperator{\bernoulli}{Bernoulli}

\DeclareMathOperator{\E}{\mathbb{E}}
\DeclareMathOperator*{\argmax}{arg\,max}

\DeclarePairedDelimiter\abs{\lvert}{\rvert}%

\DeclarePairedDelimiter\floor{\lfloor}{\rfloor}

\newcommand{\agnews}{AG-News}
\newcommand{\imdb}{IMDB}
\newcommand{\lun}{LUN}
\newcommand{\satnews}{SatNews}
\newcommand{\assassin}{Spam-assassin}

\newcommand{\ns}{Baseline}
\newcommand{\ranmask}{RanMASK}
\newcommand{\textcrs}{TextCRS}
\newcommand{\randel}{CERT-ED}%

\newcommand{\clare}{Clare}
\newcommand{\bae}{BAE-I}
\newcommand{\bertattack}{BERT-Attack}
\newcommand{\deepwordbug}{DeepWordBug}
\newcommand{\textfooler}{TextFooler}

\newcommand\theDC{}
\newcommand\boldcell{\relax\ifmmode$\egroup\fi\bfseries\boldmath\theDC}
\makeatletter
\newcolumntype{E}[3]{>{\def\theDC{\DC@{#1}{#2}{#3}}\theDC}c<{\DC@end}}
\makeatother
\newcolumntype{d}[1]{E{.}{.}{#1}}
\newcolumntype{L}[1]{>{\raggedright\let\newline\\\arraybackslash\hspace{0pt}}m{#1}}
\newcolumntype{C}[1]{>{\centering\let\newline\\\arraybackslash\hspace{0pt}}m{#1}}

%% file: sections/introduction.tex
\section{Introduction} \label{sec:intro}
\begin{figure*}[t]
    \centering
    \includegraphics[width=0.9\textwidth]{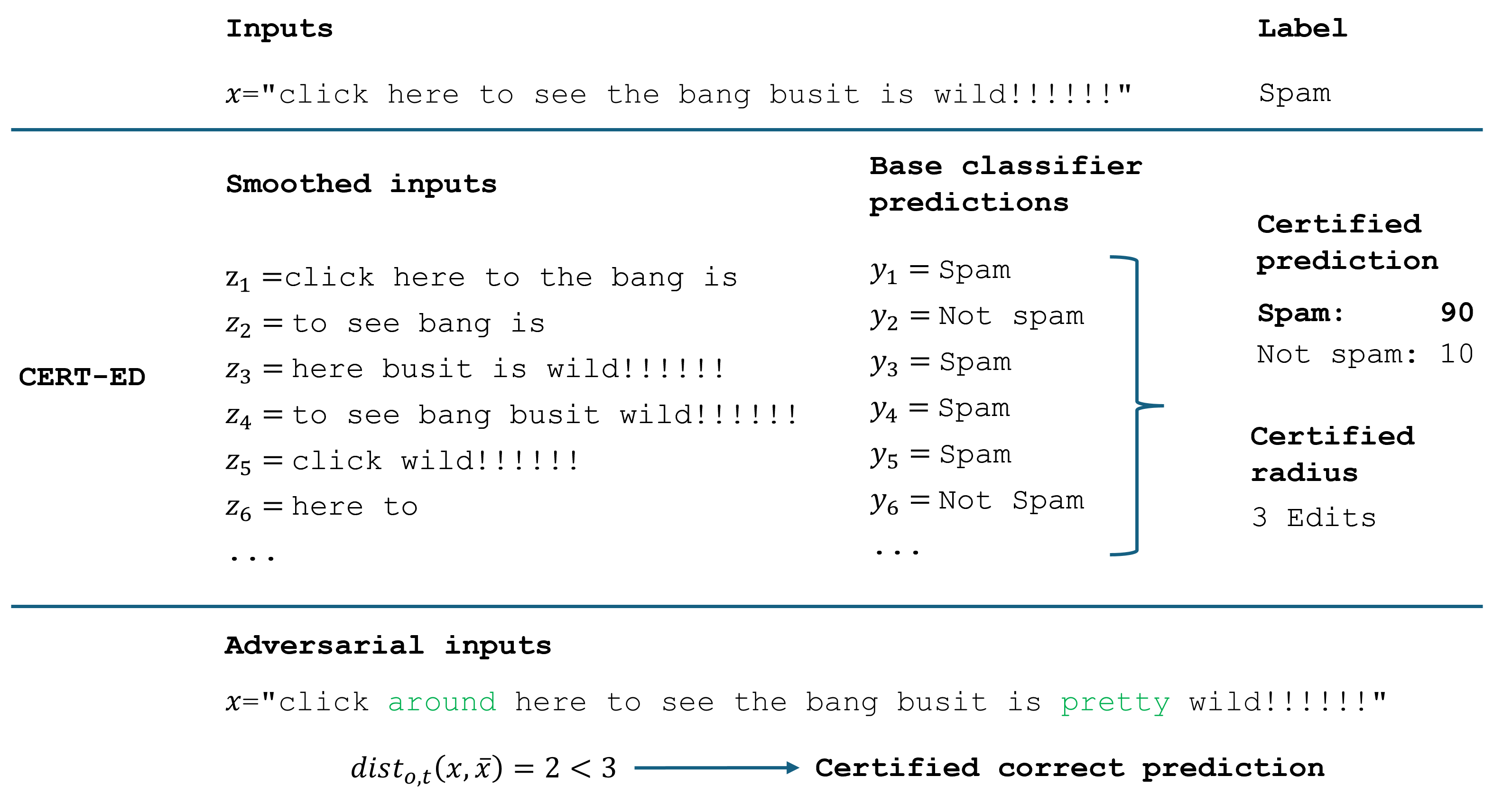}
    \caption{
      Top: Clean sample from \assassin\ dataset.
      Middle: \randel\ applied to the perturbed input to produce edit distance certified prediction of ``Spam'' and certified radius of $3$.
      Bottom: Real adversarial sample generated by, \clare\ \citep{li-etal-2021-contextualized}, against a model without \randel.
      The green words are adversarially inserted words.
      \randel\ is certifiably robust to this adversarial example as the edit distance between the clean and adversarial inputs is $2$,
      less than the certified radius.
    }
    \label{fig:sample-smoothing}
\end{figure*}

Deep nets, transformers, and other modern machine learning approaches have recently achieved significant performance on many 
natural language processing (NLP) tasks thanks to their ability to generalize to complex and unseen data.
However, the well-documented vulnerability of these models to evasion attacks (a.k.a.\ adversarial examples) raises concerns about their use in practice.
For example, numerous previous works have developed attacks that can misguide models by perturbing text at the
word-level~\citep{alzantot-etal-2018-generating,li-etal-2020-bert-attack,ren-etal-2019-generating,zang-etal-2020-word,jin2020bert,
li-etal-2021-contextualized,garg-ramakrishnan-2020-bae}, character-level~\citep{karpukhin-etal-2019-training,gao2018black,ebrahimi-etal-2018-hotflip}
or sentence-level~\citep{iyyer-etal-2018-adversarial,wang-etal-2020-cat,qi-etal-2021-mind,guo-etal-2021-gradient},
while preserving consistent semantics.

Although a wide range of defenses have been proposed against adversarial examples, they are routinely broken by subsequent attacks:
\citet{eger-benz-2020-hero} and \citet{morris2020textattack} showed that while adversarial training typically yields good robustness against a target attack,
it is less robust against unseen attacks.
Consequently, certified robustness has gained considerable interest as a result of competition between attackers and defenders,
where a classifier's prediction can be guaranteed to be invariant to a specified set of adversarial perturbations~\citep{cohen2019certified,wang-etal-2021-certified}.

Certified robustness methods have been well studied for \emph{continuous data that is fixed-dimensional} like images
\citep{wong2018provable,dvijotham2018training, mirman2018differentiable,weng2018towards,lecuyer2019certified,cohen2019certified}.
Among certification approaches, randomized smoothing \citep{lecuyer2019certified,cohen2019certified,levine2020robustness} has attained state-of-the-art performance in many tasks.
Due to the discrete nature of text inputs, however, developing randomized smoothing mechanism for NLP tasks is more challenging.
\citet{ye-etal-2020-safer} and \citet{wang-etal-2021-certified} were the first to investigate randomized smoothing under a synonym substitution threat model.
Similarly, \citet{zeng-etal-2023-certified} proposed \ranmask\ that adapts Randomized Ablation \cite{levine2020robustness} to NLP and is
provably verifiable for Hamming distance under a fixed number of word substitutions. 
However, such robustness certificates fall short in defending against general perturbations as inserting a single word
\citep{garg-ramakrishnan-2020-bae} would void any of these substitution-based certificates.
To remedy this, \citet{huang2023rsdel} proposed a method for producing certifiable predictions under edit distance perturbations, called Randomized Deletion.
However, their method is limited to binary classification tasks, and has only been applied to the malware detection domain. In this paper we address this limitation.

Our contributions are summarized as follows:
\begin{itemize}[leftmargin=*]
    \item We propose and implement CERTified Edit Distance defense (\randel), a multi-class extension of Randomized Deletion \citep{huang2023rsdel}, that can provably certify multi-class predictions for NLP classification tasks\footnote{Our implementation is available at \url{https://github.com/Dovermore/nlp-smoothing-software}.}.
    It smooths input text by adding deletion noise to produce predictions that are certifiably robust under arbitrary attacks within a computed edit distance radius $\radius$ (Figure~\ref{fig:sample-smoothing}).
    \item To compare our edit distance certificates with Hamming distance certificates used in previous work, we define certified cardinality, a discrete analogue of certified volume that has been used to compare certificates with different geometries in the vision domain.
    We evaluate \randel\ using 5 datasets and find significant improvement of over the \ranmask\ baseline, both in certified accuracy and certified cardinality.
    \item We conduct a comprehensive empirical evaluation of robustness against five state-of-the-art direct attacks and transfer attacks.
    Our results show improved robust accuracy in 20 out of 25 settings for direct attacks and 18 out of 25 settings for transfer attacks.
\end{itemize}

%% file: sections/preliminaries.tex
\section{Edit distance robustness} \label{sec:prelim}

We consider text sequence classification tasks, where a model~$f$ predicts the class $y \in \scrY$ of input text 
$\vec{x} \in \scrX$. 
For example, in fake news detection the input text is a news article and the possible classes are ``fake'' and 
``real''~\citep{rashkin-etal-2017-truth}. 
We are interested in studying robustness under an adversary that can make a bounded number of edits to the text. 

We define an \emph{edit} to be an operation that deletes (\del), inserts (\ins) or substitutes (\sub) a single 
token in the text. 
A \emph{token} is a contiguous chunk of characters---e.g., a word or sub-word.
The mapping from text to tokens is determined by the adversary's tokenizer $\tokenizer$. 
While our method is compatible with any choice of $\tokenizer$, we set $\tokenizer$ to be a whitespace tokenizer in 
our experiments for comparison with prior work on 
attacks~\citep{garg-ramakrishnan-2020-bae,jin2020bert,li-etal-2020-bert-attack,li-etal-2021-contextualized} 
and robustness~\citep{zeng-etal-2023-certified,zhang-etal-2024-random} at the word-level. 
We note that $\tokenizer$ is solely used to model the adversary's edits, and is distinct from any tokenizer that may 
appear in model $f$ itself. 

For generality, we consider adversaries whose edit operations are constrained to the set  
$\opset \subseteq \{\del, \ins, \sub\}$. 
For instance, $\opset = \{\ins, \sub\}$ for an adversary that cannot perform deletions.
Given the adversary's tokenizer $\tokenizer$ and allowed edit operations $\opset$, we measure the extent of the 
adversary's perturbation using edit distance $\dist_{\opset, \tokenizer}(\pert{\vec{x}}, \vec{x})$, which counts the 
minimum number of edits required to transform original text $\pert{\vec{x}}$ into perturbed text $\vec{x}$. 

Our objective is to design text sequence classification models that are certifiably robust under this threat model. 
Formally, given input text $\vec{x} \in \scrX$ to model $f$, we would like to guarantee that $f$'s prediction is 
unchanged even if $\vec{x}$ was modified by an adversary that made up to $\radius$ edits:
\begin{equation}
  \forall \pert{\vec{x}} \in B_{\radius}(\vec{x}; \opset, \tokenizer): f(\vec{x}) = f(\pert{\vec{x}}). 
  \label{eqn:edit-dist-cert}
\end{equation} 
Here 
\begin{equation}
  B_{\radius}(\vec{x}; \opset, \tokenizer) \coloneqq \{ \pert{\vec{x}} \in \scrX : 
    \dist_{\opset, \tokenizer}(\pert{\vec{x}}, \vec{x}) \leq \radius \} \label{eqn:edit-dist-ball}
\end{equation}
is the set of text inputs that can be transformed into $\vec{x}$ via at most $\radius$ edits. 
As is typical for randomized smoothing, we will develop mechanisms that produce a randomized 
radius $\radius$ given input sequence $\vec{x}$, such that with some chosen high probability 
at least $1-\alpha$, this radius is a valid certificate at $\vec{x}$.

%% file: sections/method.tex
\section{Certified robustness via randomized smoothing} \label{sec:method}

Our approach for achieving certified robustness under bounded edit distance perturbations is based on 
randomized smoothing. 
Specifically, we apply the randomized smoothing mechanism of \citet{huang2023rsdel}, which was originally formulated 
for binary classification of generic sequences. 
We review the mechanism in a text classification context in Section~\ref{sec:randel} and
propose, \randel, a certifiably robust extension to the multi-class setting in Section~\ref{sec:randel-cert}.
Our derived certificates cover attacks \bae~\citep{garg-ramakrishnan-2020-bae} and 
\clare~\citep{li-etal-2021-contextualized} not covered by prior work.

\subsection{Randomized deletion smoothing} \label{sec:randel}

Randomized smoothing has emerged as a general purpose method for constructing certifiably robust 
classifiers~\citep{kumari2023trust}. 
Consider a base classifier $\base{f} \colon \scrX \to \scrY$ and a randomized mechanism $\phi \colon \scrX \to P(\scrX)$ 
that generates perturbed inputs. In the following we construct a smoothed classifier $\smooth{f}$ 
that assigns a probabilistic  score to class $y$ given input $\vec{x}$: %
\begin{equation}
  p_y(\vec{x}) = \E_{\vec{z} \sim \phi(\vec{x})} [\ind{\base{f}(\vec{z}) = y}].
  \label{eqn:smooth-score}
\end{equation}
The smoothed classifier's prediction is then the class with the highest score: 
$\smooth{f}(\vec{x}) = \argmax_{y \in \scrY} p_y(\vec{x})$.

\citet{huang2023rsdel} instantiate randomized smoothing with a deletion mechanism for sequences that achieves 
certified edit distance robustness. 
In the text domain, we apply their deletion mechanism at the level of tokens determined by tokenizer $\tokenizer$. 
Specifically, given input text $\vec{x} \in \scrX$ containing $n$ tokens, the deletion mechanism generates $n$ 
indicator variables $\vec{\mdel} = (\mdel_1, \ldots, \mdel_{n})$ where 
$\mdel_i \overset{\mathrm{iid}}{\sim} \bernoulli(p_\del)$. 
The perturbed text is then obtained by deleting any token $i$ for which $\mdel_i = 1$ and keeping  
the remaining tokens (in order).

\paragraph{Practicalities}
Exact computation of the scores in~\eqref{eqn:smooth-score} scales exponentially in the number of tokens~$n$. 
We therefore follow standard practice in randomized smoothing and obtain upper\slash lower confidence bounds on the 
scores using Monte Carlo sampling. 
When constructing a smoothed classifier $\smooth{f}$, we fine-tune the base classifier $\base{f}$ on inputs 
perturbed by the mechanism $\phi$, as this results in better performance. 

\subsection{\randel: Multi-class edit distance certification} \label{sec:randel-cert}

We extend the edit distance certificate of \citet{huang2023rsdel} to the multi-class setting. 
We refer to randomized deletion smoothing with this new certificate as \randel. 
We discuss the setup and assumptions here and present two key results, which demonstrate how a certificate can be 
obtained for a bounded Levenshtein (edit) distance adversary whose edit operations $\opset$ are unconstrained. 
All proofs are provided in Appendix~\ref{sec:proofs}.

We adopt a standard setup for certification of smoothed classifiers. 
Given input text $\vec{x}$, we let $y = \argmax_{c} p_c(\vec{x})$ be the class that achieves the highest score and 
$y' = \argmax_{c \neq y} p_c(\vec{x})$ be the class that achieves the second-highest score. 
We assume that $\mu_y \leq p_y(\vec{x})$ is a lower bound on the highest score and $\mu_{y'} \geq p_{y'}(\vec{x})$ 
is an upper bound on the second highest score. 
Apart from this minimal information, we do not assume any knowledge of the base classifier $\base{f}$.
We begin by obtaining upper and lower bounds on the smoothed classifier's score at a neighboring input 
$\pert{\vec{x}}$.

\begin{theorem}[name=General pairwise certificate,restate=pairwisecert] \label{thm:del-cert-pairwise}
  Consider a pair of text inputs $\vec{x}, \pert{\vec{x}} \in \scrX$. 
  Suppose $\pert{\vec{x}}$ can be transformed into $\vec{x}$ using a minimal number of edit operations by 
  deleting $n_\del$ tokens, inserting $n_\ins$ tokens and substituting $n_\sub$ tokens---i.e., 
  $\dist_{\opset,\tokenizer}(\pert{\vec{x}}, \vec{x}) = n_\sub + n_\ins + n_\del$. 
  Then the smoothed classifier's scores for any class $y \in \scrY$ satisfy
  \begin{align*}
    & p_\del^{n_\del - n_\ins} \left(p_y(\vec{x}) - 1 + p_\del^{n_\sub + n_\ins}\right) 
      \leq p_y(\pert{\vec{x}}) \\
    & \qquad \leq p_\del^{n_\del - n_\ins} p_y(\vec{x}) + 1 - p_\del^{n_\sub + n_\del}.
  \end{align*}
\end{theorem}

The above result is not immediately useful on its own. 
However, it can be used to derive edit distance certificates under various constraints on the number of edit operations 
of each type ($n_\del, n_\ins, n_\sub$) the adversary can perform. 
Below we present a Levenshtein distance certificate, which covers an adversary that can perform up to $\radius$  
edits of any type (insertions, deletions or substitutions).

\begin{theorem}[name=Levenshtein distance certificate,restate=levcert] \label{thm:del-cert-lev}
  Consider a text input $\vec{x} \in \scrX$ for which a lower bound on the smoothed classifier's highest score $\mu_y$ 
  and an upper bound on the smoothed classifier's runner-up score $\mu_{y'}$ satisfy $\mu_y \geq \mu_{y'}$. 
  Then the smoothed classifier predicts $y$ for any neighboring text input $\pert{\vec{x}} \in \scrX$ such that 
  $\dist_{\opset,\tokenizer}(\pert{\vec{x}},\vec{x}) \leq \radius$ with $\opset = \{\del, \ins, \sub\}$ and
  $\radius = \floor{\log_{p_\del} \frac{1}{2}(2 + \mu_{y'} - \mu_y)}$. 
  If the upper and lower bounds hold jointly with confidence $1 - \alpha$, then the certificate holds with 
  probability $1 - \alpha$.
\end{theorem}

This result is readily adapted for adversaries that are constrained in the kinds of edit operations they can perform. 
We provide certificates for seven constrained settings in Table~\ref{tbl:del-cert-radii}.

\begin{table}
  \centering
  \begin{tabular}{lllc}
    \toprule
    \multicolumn{3}{c}{Adversary's ops}      & \\
    \cmidrule{1-3}
    \del       & \ins       & \sub       & Certified radius ($\downarrow$) \\
    \midrule 
    \checkmark & \checkmark & \checkmark & $\floor*{\log_{p_\del}\frac{2 + \mu_{y'} - \mu_y}{2}}$ \\\addlinespace[0.2em]
    \checkmark &            & \checkmark & $\floor*{\log_{p_\del}\frac{2 + \mu_{y'} - \mu_y}{2}}$ \\\addlinespace[0.2em]
               & \checkmark & \checkmark & $\floor*{\log_{p_\del}\frac{2 + \mu_{y'} - \mu_y}{2}}$ \\\addlinespace[0.2em]
               &            & \checkmark & $\floor*{\log_{p_\del}\frac{2 + \mu_{y'} - \mu_y}{2}}$ \\\addlinespace[0.2em]
    \checkmark & \checkmark &            & $\floor*{\log_{p_\del}\frac{1}{1 - \mu_{y'} + \mu_y}}$ \\\addlinespace[0.2em]
    \checkmark &            &            & $\floor*{\log_{p_\del}\frac{1}{1 - \mu_{y'} + \mu_y}}$ \\\addlinespace[0.2em]
               & \checkmark &            & $\floor*{\log_{p_\del}(1 + \mu_{y'} - \mu_y)}\vphantom{\frac{1}{1}}$ \\
    \bottomrule
  \end{tabular}
  \caption{Certified radii as a function of the types of edit operations the adversary can perform.}
  \label{tbl:del-cert-radii}
\end{table}

%% file: sections/evaluation.tex
\section{Experiments} \label{sec:eval}
To evaluate the effectiveness of our methods, we train and evaluate models with \randel\ certification across a variety of English datasets and
compare observed performance against RanMASK~\citep{zeng-etal-2023-certified}. %
We show that our method uniformly dominates \ranmask\ in certified radius and accuracy for 4 out of 5 datasets in Section~\ref{sec:eval-cert}.
We then conduct 5 direct and 5 transfer attacks on the certified models and show that \randel\ is also empirically
more robust than \ranmask\ in 38 out of 50 settings in Section~\ref{sec:eval-attack}.
Though \randel\ is empirically more robust than \ranmask\ against word substitution and character-level attacks,
the performance is more mixed for attacks that induce edit distance perturbations.
Finally, we also report a 3~times speedup of our method compared to \ranmask\ in Appendix~\ref{app-sec:eval-efficiency}.

\paragraph{Datasets}
We present results on five diverse datasets listed in Table~\ref{tbl:data-summary}.
All datasets are partitioned into training, validation and test sets.
\agnews~\citep{zhang2015character} and \imdb~\citep{maas-EtAl:2011:ACL-HLT2011} are
standard datasets for topic classification and sentiment analysis respectively, that
have been used for evaluation in prior work~\citep{zeng-etal-2023-certified}.
However, since adversaries may lack incentive to target these models, we follow recommendations of \citet{chen-etal-2022-adversarial},
and consider \assassin~\citep{chen-etal-2022-adversarial} for spam detection, and
\satnews~\citep{yang-etal-2017-satirical} and \lun~\citep{rashkin-etal-2017-truth,chen-etal-2022-adversarial} for unreliable news detection,
which are arguably more attractive targets for attackers\footnote{These datasets might contain offensive or inappropriate languages due to their adversarial nature}.
We collect all data from HuggingFace Datasets\footnote{\url{https://github.com/huggingface/datasets}} and the \texttt{AdvBench}
repository\footnote{\url{https://github.com/thunlp/Advbench}}~\citep{chen-etal-2022-adversarial}.

\begin{table}[h]
  \begin{center}
    \small
    \begin{NiceTabular}{l@{}rrrr}
      \toprule
                &           & \multicolumn{3}{c}{Number of samples}                     \\
      \cmidrule(lr){3-5}
      Dataset   & Avg words & Train                                 & Valid   & Test    \\
      \midrule
      \agnews   & 37.8      & 108\,000                              & 12\,000 & 7\,600  \\
      \midrule
      \imdb     & 231.2     & 22\,500                               & 2\,500  & 25\,000 \\
      \midrule
      \assassin & 228.2     & 2\,152                                & 239     & 2\,378  \\
      \midrule
      \lun      & 269.9     & 13\,416                               & 1\,490  & 6\,454  \\
      \midrule
      \satnews  & 384.8     & 22\,738                               & 2\,526  & 7\,202  \\
      \bottomrule
    \end{NiceTabular}
  \end{center}
  \caption{Summary of datasets.}
  \label{tbl:data-summary}
\end{table}

\paragraph{Models}
We use the Hugging Face Transformers library~\citep{wolf-etal-2020-transformers} to load a pre-trained
RoBERTa model~\citep{liu2019roberta} as a base classifier for \randel\ and as a non-certified baseline.
We include \ranmask~\citep{zeng-etal-2023-certified} as a certified baseline for Hamming distance. This baseline covers a bounded number of word substitutions only.
Since \ranmask\ is also based on randomized smoothing, albeit with a different masking mechanism, we use the same base classifier, training procedure and parameter settings as \randel\ where possible.
We use \emph{perturbation strength} to refer to the deletion rate parameter $p_\del$ of \randel\ and the masking rate parameter $p_\mask$ of \ranmask.
We note that \textcrs\ \citep{zhang2024textcrs} claims to certify against limited word substitution, insertion,
and deletion operations of a single type, however we do not include it as a baseline since the resulting edit distance certificates are vacuous ($\radius=0$) for text greater than
$2$ words in length (see Appendix~\ref{app-sec:textcrs}).

For both \randel\ and \ranmask, we use a white-space tokenizer for smoothing and fine-tune the base RoBERTa model on the training set where inputs are perturbed by the corresponding smoothing mechanism.
We report our training parameter settings in Table~\ref{tbl:model-parameters} of Appendix~\ref{app-sec:parameters}.

\input{sections/evaluation-certify.tex}
\input{sections/evaluation-attack.tex}

%% file: sections/evaluation-certify.tex
\subsection{Certified accuracy and robustness} \label{sec:eval-cert}
\paragraph{Setup}
Our first set of experiments compares the accuracy and certificates generated by \randel\ and \ranmask.
Following prior work \citep{huang2023rsdel,zeng-etal-2023-certified}, we 
use $1000$ Monte Carlo samples for prediction, $4000$ samples for estimating the certified radius and 
a confidence level of $95\%$.
In the unlikely case where the prediction and certificate disagree, we report the prediction, and report a certified radius of zero.
Table~\ref{tbl:cr-statistics-small} presents a summary of the results, where we highlight the best of three perturbation strengths (80\%, 90\% and 95\%) for each dataset.
Results for all three perturbation strengths (80\%, 90\% and 95\%) are provided in Table~\ref{tbl:cr-statistics} of Appendix~\ref{app-sec:cr-statistics}.

\begin{table}[h]
  \begin{center}
    \small
    \begin{NiceTabular}{
        @{\hskip 5.5pt}l@{}r
        r
        r
        r@{\hskip 5.5pt}
      }[colortbl-like]
      \toprule
               &                  & Clean            & Median & Median              \\
      Model    & $p_\del/p_\mask$ & Accuracy         & CR     & $\log$ CC           \\
      \midrule
      \multicolumn{5}{c}{\rowcolor{gray!10} \agnews\ dataset (avg length 37.84)}    \\
      \midrule
      \ns      & ---              & 94.84\%          & ---    & ---                 \\
      \midrule
      \ranmask & 80\%             & \textbf{93.91\%} & 2      & 12.25               \\
      \randel  & 80\%             & 93.33\%          & 2      & \textbf{12.65}      \\
      \midrule
      \multicolumn{5}{c}{\rowcolor{gray!10} \imdb\ dataset (avg length 231.16)}     \\
      \midrule
      \ns      & ---              & 93.47\%          & ---    & ---                 \\
      \midrule
      \ranmask & 90\%             & 86.87\%          & 2      & 14.02               \\
      \randel  & 90\%             & \textbf{88.26\%} & 2      & \textbf{14.49}      \\
      \midrule
      \multicolumn{5}{c}{\rowcolor{gray!10} \assassin\ dataset (avg length 228.16)} \\
      \midrule
      \ns      & ---              & 98.02\%          & ---    & ---                 \\
      \midrule
      \ranmask & 90\%             & 97.65\%          & 6      & 37.49               \\
      \randel  & 90\%             & \textbf{97.81\%} & 6      & \textbf{38.66}      \\
      \midrule
      \multicolumn{5}{c}{\rowcolor{gray!10} \lun\ dataset (avg length 269.93)}      \\
      \midrule
      \ns      & ---              & 99.16\%          & ---    & ---                 \\
      \midrule
      \ranmask & 90\%             & 97.91\%          & 6      & 34.51               \\
      \randel  & 90\%             & \textbf{98.28\%} & 6      & \textbf{37.94}      \\
      \midrule
      \multicolumn{5}{c}{\rowcolor{gray!10} \satnews\ dataset (avg length 384.84)}  \\
      \midrule
      \ns      & ---              & 94.22\%          & ---    & ---                 \\
      \midrule
      \ranmask & 95\%             & 90.10\%          & 7      & 47.09               \\
      \randel  & 95\%             & \textbf{92.07\%} & 8      & \textbf{54.76}      \\
      \bottomrule
    \end{NiceTabular}
  \end{center}
  \caption{
    Key certification results drawn from Table~\ref{tbl:cr-statistics}. 
    All metrics are computed on the entire test set.
    ``Median CR'' is the median certified radius and ``median $\log$ CC'' is the base-10 
    logarithm of the median certified cardinality.
    The certified cardinality is exact for \ranmask, however a lower bound is used for \randel.
    \randel\ outperforms \ranmask\ in terms of certified accuracy for 4 out of 5 datasets and specifically
    excels on datasets with longer average length.
    Highlighted values are the better of the two smoothed mechanisms.
  }
  \label{tbl:cr-statistics-small}
\end{table}

\paragraph{Clean accuracy}
To assess performance in a non-adversarial setting, we report accuracy on the clean (unperturbed) test set. 
For both \randel\ and \ranmask, we observe a slight degradation in clean accuracy of 0.2--5.2\% across all datasets compared to the non-smoothed baseline,
where the highest degradation is seen on the \imdb\ dataset.
This may be due to the fact that \imdb\ is a sentiment classification task, where sensitivity
to small perturbations such as names \cite{prabhakaran-etal-2019-perturbation} and the presence or absence of single words (e.g., ``not'') is more likely.
\randel\ outperforms \ranmask\ in both clean accuracy and the size of the certificate for the four datasets with longer text.
While \randel\ fuzzes the input length, \ranmask\ preserves it,
and may therefore have an advantage for short text where the length conveys information.
However, for datasets with longer text like \lun\ and \satnews,
\ranmask\ may introduce too many masking tokens and confuse the base model.

\begin{figure}
  \centering
  \includegraphics[width=\linewidth]{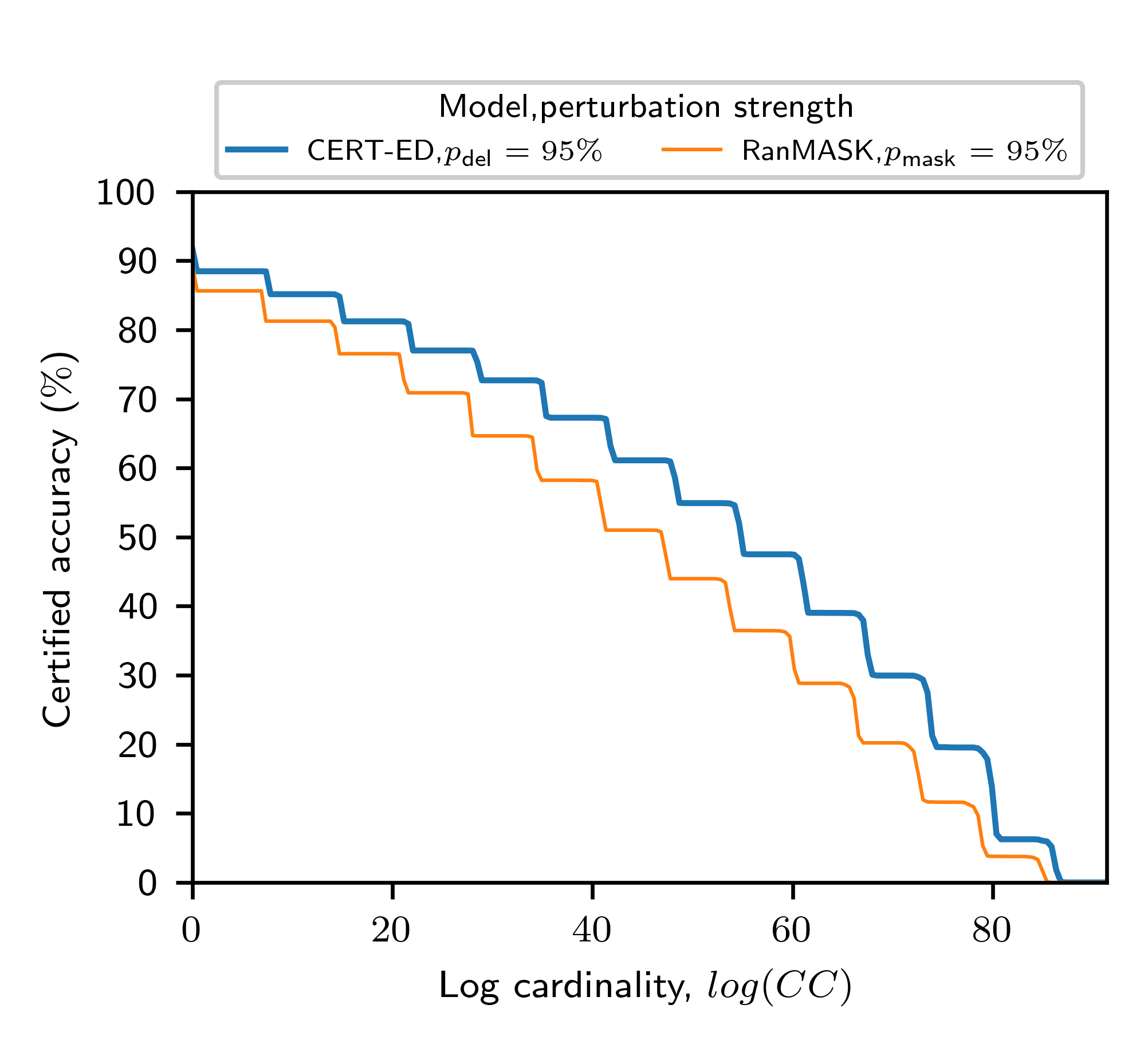}
  \caption{
    Certified accuracy for \randel\ and \ranmask\ as a function of the log-cardinality of the certificate for the \satnews\ dataset.
    We see that \randel\ certifies a set up to $10^{10}$ times larger than \ranmask\ for the same accuracy.
  }
  \label{fig:certified_accuracy-satnews}
\end{figure}

\paragraph{Certified accuracy and certified cardinality}
We report certified radius (CR) as a measure of robustness, where a larger radius indicates robustness to larger perturbations. 
However, the certified radius is not comparable between \randel\ and \ranmask, as the metric used to define the certificate for \randel\ is Levenshtein (edit) distance whereas the metric for \ranmask\ is Hamming distance. 
As an alternative, we therefore also measure the cardinality of the certificate, that is the number of perturbed textual inputs it contains, which we refer to as the certified cardinality (CC).
This is analogous to certified volume, used in prior work in the vision domain \citep{pfrommer2023projected}. 
For a given radius and vocabulary size\footnote{
  We assume a vocabulary size of $50\,265$ (matching RoBERTa), although this is a property of the threat model.
}, we compute the cardinality exactly for \ranmask\ and use a lower bound for \randel\ based on a result of \citet[Fact~17]{charalampopoulos2020unary}. 
We note that it is possible to compute the exact cardinality for a Levenshtein distance certificate using a Levenshtein automaton \citep{touzet2016on}, however the computation is expensive and only improves on the lower bound by approximately one order of magnitude. 
Despite the under-estimated certified cardinality of \randel, we find it dominates \ranmask\ across all datasets, with a widening gap for datasets with longer text. 
To assess trade-offs between robustness and accuracy, we plot the certified accuracy as a function of the log-cardinality. 
The certified accuracy at a given log-cardinality $c$ is the fraction of instances in the test set for which the model's prediction is correct and the cardinality of the certificate is at least $c$. 
Figure~\ref{fig:certified_accuracy-satnews} plots the certified accuracy for the \satnews\ dataset,
demonstrating an improvement up to \emph{$10^{10}$ times} in the certified cardinality for \randel\ compared to \ranmask.
Due to space constraints, we present certified accuracy plots for the remaining datasets in Appendix~\ref{app-sec:ca-plots},
where a strict domination of \randel\ over \ranmask\ can be observed for~4 out of~5 datasets.

%% file: sections/evaluation-attack.tex
\subsection{Empirical robustness} \label{sec:eval-attack}

\paragraph{Attack setup}
We evaluate the empirical robustness of \randel\ and baselines using a modified version of 
\texttt{TextAttack}\footnote{\url{https://github.com/QData/TextAttack}}~\citep{morris2020textattack} and attack recipes implemented by \citet{zhang2024textcrs}.
We select five representative attacks that cover a variety of perturbations: \clare~\citep{li-etal-2021-contextualized}, \bae~\citep{garg-ramakrishnan-2020-bae},
\bertattack~\citep{li-etal-2020-bert-attack}, \textfooler~\citep{jin2020bert}
and \deepwordbug~\citep{gao2018black}. 
We describe these attacks further in Section~\ref{sec:related-attacks}. 
We randomly select 1000 samples from the test set to evaluate the robustness of the models.
For \randel\ and \ranmask, we estimate the prediction using a Monte Carlo sample size of 100 to speed up the attack process.
We impose a \emph{10 minute timeout} for each attack and treat timeout as a failed attack.
As \randel\ is about 3~times faster than \ranmask\ (Appendix~\ref{app-sec:eval-efficiency}) for prediction and certification,
this puts \randel\ at a disadvantage as attacks against it may use up to a 3~times as many queries.
For \clare, due to the excessive amount of querying, we also limit the maximum number of queries to $10\,000$.
Further details on our categorization of attack outcomes are provided in Appendix~\ref{app-sec:eval-attack}.

\paragraph{Threat models}
We consider two distinct threat models: \emph{direct attacks} which have access to the model's confidence;
and \emph{transfer attacks} for which attacks are generated against the non-certified baseline and if successful, transferred to the target model (Appendix~\ref{app-sec:eval-attack}).
This means the robust accuracy for the non-certified baseline will always be 0\%.
We report both the \emph{clean accuracy (ClA)} on original instances, and \emph{robust accuracy (RoA)} on attacked instances.

\begin{table*}[h]
  \begin{center}
    \small
    \begin{NiceTabular}{
        l
        c
        |c
        c
        c
        c
        c
      }[colortbl-like]
      \toprule
               & \textbf{Clean}
               & \textbf{\clare}
               & \textbf{\bae}
               & \textbf{\bertattack}
               & \textbf{\textfooler}
               & \textbf{\deepwordbug}                                                                                      \\
      \textbf{Method}
               & \textbf{Accuracy \%}
               & \textbf{RoA\%}
               & \textbf{RoA\%}
               & \textbf{RoA\%}
               & \textbf{RoA\%}
               & \textbf{RoA\%}
      \\
      \midrule
      \multicolumn{7}{c}{\rowcolor{gray!10} \agnews\ dataset}                                                               \\
      \midrule
      \ns      & \textbf{94.50}        & 29.50          & 58.50          & 29.00          & 22.10          & 43.50          \\
      \ranmask & 91.70                 & \textbf{77.30} & \textbf{79.60} & 45.10          & 48.50          & 47.70          \\
      \randel  & 92.00                 & 67.90          & 78.10          & \textbf{54.40} & \textbf{56.10} & \textbf{56.90} \\
      \midrule
      \multicolumn{7}{c}{\rowcolor{gray!10} \imdb\ dataset}                                                                 \\
      \midrule
      \ns      & \textbf{94.80}        & 50.30          & 28.10          & 9.20           & 8.00           & 30.00          \\
      \ranmask & 86.90                 & \textbf{85.50} & \textbf{81.90} & 52.40          & 53.00          & 45.90          \\
      \randel  & 87.90                 & 82.30          & 71.10          & \textbf{55.30} & \textbf{54.00} & \textbf{49.90} \\
      \midrule
      \multicolumn{7}{c}{\rowcolor{gray!10} \assassin\ dataset}                                                             \\
      \midrule
      \ns      & \textbf{97.59}        & 92.05          & 94.80          & 54.55          & 36.20          & 80.20          \\
      \ranmask & 97.00                 & 95.40          & \textbf{96.00} & 93.50          & 93.94          & 92.10          \\
      \randel  & 97.20                 & \textbf{95.20} & 94.90          & \textbf{94.70} & \textbf{95.45} & \textbf{94.30} \\
      \midrule
      \multicolumn{7}{c}{\rowcolor{gray!10} \lun\ dataset}                                                                  \\
      \midrule
      \ns      & \textbf{99.50}        & 91.50          & 80.10          & 57.20          & 64.60          & 78.50          \\
      \ranmask & 98.70                 & 93.80          & 96.40          & 88.80          & 90.50          & 90.50          \\
      \randel  & 99.30                 & \textbf{95.30} & \textbf{96.60} & \textbf{92.40} & \textbf{93.90} & \textbf{93.30} \\
      \midrule
      \multicolumn{7}{c}{\rowcolor{gray!10} \satnews\ dataset}                                                              \\
      \midrule
      \ns      & \textbf{95.50}        & 84.50          & 58.60          & 51.30          & 50.70          & 61.40          \\
      \ranmask & 91.20                 & 91.20          & 91.20          & 74.50          & 80.70          & 80.80          \\
      \randel  & 93.20                 & \textbf{93.20} & \textbf{91.70} & \textbf{83.30} & \textbf{83.50} & \textbf{81.80} \\
      \bottomrule
    \end{NiceTabular}
  \end{center}
  \caption{
    Empirical attack results against \ns, \ranmask\ and \randel\ under direct attacks.
    Both \ranmask\ and \randel\ use a $90\%$ perturbation strength.
    Highlighted values are the best in each column for that dataset.
    \randel\ outperforms \ranmask\ in \emph{all} substitution and character-level attacks.
    For edit distance attacks, the performance is mixed, with \randel\ performing better for datasets with longer text.
  }
  \label{tbl:direct-attack}
\end{table*}

\paragraph{Results}
We present direct attack results in Table~\ref{tbl:direct-attack}. 
Compared to the non-certified baseline, \randel\ sacrifices a small amount of clean accuracy to achieve a significant improvement in robust accuracy across all settings.
For word substitution and character-level attacks (\bertattack, \textfooler, \deepwordbug),
\randel\ uniformly outperforms \ranmask\ in all datasets by up to $8.8\%$ robust accuracy.
Surprisingly, we find that \randel\ is less effective against edit distance attacks (\clare\ and \bae),
performing well on datasets with longer input sizes, but worse on \agnews\ and \imdb.
We attribute this to the difference in query efficiency between \ranmask\ and \randel.
For \clare\ on the \agnews\ dataset, the mean number of queries to \ranmask\ and \randel\ are $4\,253$ and $8\,891$, respectively,
meaning \randel\ is subject to a stronger attack than \ranmask.
Due to the longer text size, all \clare\ attacks timed out on the \satnews\ dataset for \randel\ and \ranmask.

%% file: sections/related.tex
\section{Related work} \label{sec:related}
Despite the rapid advancement of neural networks for many important tasks, 
their vulnerability to adversarial examples has long been known \citep{szegedy2014intriguing,
goodfellow2015explaining} and continues to be a liability for current large-scale models \citep{qi2024visual,
raina2024llmasajudge}. 
While early work focused on the vision domain, where attacks can leverage continuous optimizers, methods have 
been proposed to generate adversarial examples in the language domain across a variety of tasks 
including text classification \citep{alzantot-etal-2018-generating,garg-ramakrishnan-2020-bae,li-etal-2020-bert-attack}
and machine translation \citep{zhang-etal-2021-crafting,belinkov2018synthetic}. 
These adversarial examples are especially concerning when applied to high stakes tasks such as
content moderation and fake news detection \citep{chen-etal-2022-adversarial}.
Alongside research on attacks, certified defenses have been proposed as a strong counter-measure that can guarantee that the most pernicious 
adversarial examples do not exist within a specified threat model \citep{lecuyer2019certified,cohen2019certified}.
However, their application in NLP has been limited, in part due to the discrete nature of text 
\cite{ye-etal-2020-safer,wang-etal-2021-certified,zeng-etal-2023-certified,wang-etal-2023-robustness,zhang-etal-2024-random}.

\subsection{Attacks}\label{sec:related-attacks}

Character-level attacks focus on perturbing characters within words to maintain the imperceptibility of attacks to human
inspection. \citet{karpukhin-etal-2019-training} augmented training data with orthographic noise (character-level insertions,
substitutions, deletions, and swaps) to improve the robustness of machine translation models. 
The DeepWordBug~\citep{gao2018black} and HotFlip~\citep{ebrahimi-etal-2018-hotflip} attacks rank character and token 
importance in the input text, then greedily search for perturbations using the importance ranking to achieve 
misclassification.

Word-level attacks generate adversarial examples by perturbing individual words in text to maximize the model loss
while preserving the semantics and syntax of the original sentence. Synonym substitution is one of the most common
approaches for maintaining semantic consistency. These attacks typically start by ranking the words by their importance
and then sequentially substituting words with synonyms generated using a similarity metric
\citep{alzantot-etal-2018-generating,li-etal-2020-bert-attack,ren-etal-2019-generating,zang-etal-2020-word,jin2020bert}.
While powerful, these attacks do not fully represent the capabilities of an adversary, as they do not consider
insertion or deletion of words. Subsequent work has investigated edit distance-constrained adversarial attacks. Both
\clare~\citep{li-etal-2021-contextualized} and \bae~\cite{garg-ramakrishnan-2020-bae} explored the use of edit
perturbations beyond substitution and found them successful against NLP models.

Our method \randel\ can certify robustness against word- and character-level attacks by setting $\tokenizer$ to be a 
white-space or character-level tokenizer, accordingly.

Unlike word- or character-level attacks, that perform a limited number of localized perturbations to preserve semantics,
sentence-level attacks adversarially paraphrase entire sentences, typically using LLMs
\citep{iyyer-etal-2018-adversarial,wang-etal-2020-cat,qi-etal-2021-mind},
or using a parametrized adversarial distribution that enables gradient-based search \citep{guo-etal-2021-gradient}.

\subsection{Defenses}
Starting with adversarial training \citep{goodfellow2015explaining}, numerous empirical defense methods have been proposed
to improve robustness of NLP models \citep{ren-etal-2019-generating,zang-etal-2020-word,wang2021adversarial,zhu2020freeLB,ivgi-berant-2021-achieving,li-etal-2020-bert-attack}.
In the language domain, adversarial training perturbs inputs either in the text space %
\citep{ren-etal-2019-generating,jin2020bert,zang-etal-2020-word,li-etal-2020-bert-attack,ivgi-berant-2021-achieving,wang2021adversarial}
or in the embedding space using bounded adversarial noise \citep{miyato2017adversarial,zhu2020freeLB}.
Although empirical defenses may provide excellent robustness against attacks they are tailored for,
they cannot guarantee effectiveness against an adaptive attacker.

To mitigate issues with empirical defenses, certified defenses aim to provide a robustness guarantee against arbitrary attacks within a specified threat model.
\citet{huang-etal-2019-achieving} and \citet{jia-etal-2019-certified} first used interval bound propagation to certify robustness under synonym substitutions. 
More recently, works have applied randomized smoothing to achieve certified robustness under synonym or word substitution threat models
\citep{ye-etal-2020-safer,wang-etal-2021-certified}.
Despite the use of insertion and deletion operations in published attacks, achieving certified robustness against these operations has not been well-studied. 
To date, only \textcrs\ \citep{zhang2024textcrs} has made progress on this front. 
\textcrs\ certifies against word-level permutations and perturbations in the embedding space, which can provide 
robustness guarantees against limited word-level substitutions, deletions or insertions. 
However, \textcrs\ only partially covers the edit distance ball at a given radius, meaning its certificates are 
vacuous for our threat model (see Appendix~\ref{app-sec:textcrs}).
In the malware detection domain, \citet{huang2023rsdel} proposed a deletion-based mechanism to achieve certified 
edit distance robustness for malware binary classification models. 
Our proposed method \randel\ adapts their mechanism for the language domain and extends the certificate to  
support multi-class classification.

%% file: sections/conclusion.tex
\section{Conclusion} \label{sec:conc}

In this work, we investigated certified robustness for natural language classification tasks,
where adversaries can perturb input text by adding, deleting, or substituting words.
We adapted randomized deletion smoothing \citep{huang2023rsdel} to the language domain,
and derived an edit distance robustness certificate for the multi-class setting.
We refer to our certified method as \randel\ and conducted comprehensive experiments on five datasets.
Our results show that \randel\ outperforms the existing randomized smoothing method (\ranmask) for
word substitution robustness in terms of both accuracy and certified cardinality on $4$ out of $5$ datasets.
Our method also excels in robustness against direct and transfer attacks, demonstrating significant improvements over existing methods.

%% file: sections/ethics.tex
\section{Ethical considerations} \label{sec:ethics}

This study focuses on enhancing the robustness of NLP models.
Although adversarial examples are generated during the research, their use is strictly for evaluation purposes. We also
acknowledge the assistance of ChatGPT and GitHub Copilot for scaffolding code in released artifacts.

%% file: sections/limitation.tex
\section{Limitations}
Robustness certification aims to measure the risk of adversarial examples,
while randomized smoothing provides both certification and mitigation against such attacks.
However, our results show that at higher smoothing levels,
randomized smoothing can reduce the benign accuracy compared to undefended models.
While our experiments are comprehensive and cover a wide range of datasets, attacks and threat models,
results might differ on other base model architectures and natural language tasks.
Although our edit distance threat model covers a boarder range of attacks than prior work on certifications for NLP, and threat models better aligned for sequence data than the bounded $\ell_p$-norm threat models popular for image research,
attackers may opt to make many edits to input data, and in some NLP tasks input semantics could be changed with few edits. 
For example, sentiment analysis can be local in nature, relying on the sentiment of a single adjective. By contrast, tasks like fake news detection typically rely on more global features to distinguish task classes.
Lastly, while our approach is more scalable than alternative certification strategies, it does introduce computational overheads.

%% file: appendices/proofs.tex
\section{Proofs for Section~\ref{sec:randel-cert}} \label{sec:proofs}

In this appendix, we provide proofs of the certification results presented in Section~\ref{sec:randel-cert}. 
Our proofs follow \citep{huang2023rsdel}, however we additionally provide an upper bound on the smoothed 
classifier's score, which is needed to support certification of multi-class classifiers. 

We begin by defining notation to express the deletion mechanism symbolically. 
Recall from Section~\ref{sec:randel} that $\vec{\mdel} = (\mdel_1, \ldots, \mdel_n)$ is a vector of deletion indicator 
variables for input text $\vec{x}$ containing $n = \abs{\tokenizer(\vec{x})}$ tokens, where $\mdel_i = 1$ if the 
$i$-th token is to be deleted and $\mdel_i = 0$ otherwise. 
The space of possible deletion indicators for input text $\vec{x}$ is denoted 
$\scrE(\vec{x}) = \{0, 1\}^{\abs{\tokenizer(\vec{x})}}$.
We let $q(\vec{\mdel}) = \prod_i p_\del^{\mdel_i} (1 - p_\del)^{1 - \mdel_i}$ denote the (Bernoulli) 
probability mass for a given $\vec{\mdel}$. 
We write $\apply(\vec{x}, \vec{\mdel})$ to denote the resultant text after deleting the 
tokens referenced in $\vec{\mdel}$ from text $\vec{x}$.   

Using this notation, we can express the smoothed classifier's score defined in \eqref{eqn:smooth-score} as a sum 
over the space of deletion indicator variables:
\begin{gather}
  p_y(\vec{x}) = \sum_{\medit \in \scrE(\vec{x})} s(\medit, \vec{x}) \\ 
    \text{with} \ s(\medit, \vec{x}) = q(\medit) \ind{\base{f}(\apply(\vec{x}, \medit)) = y}.
\end{gather}

The first step in our analysis is to identify a correspondence between deletions for neighboring text. 
Let
\begin{align*}
  \sqsubseteq \: = \{ \,
      (\vec{\mdel}, \vec{\mdel}')  \in \scrE(\vec{x}) \times \scrE(\vec{x}) : 
      \forall i, \mdel_i \leq \mdel_i' 
  \, \}.
\end{align*}
be a partial order on the space of deletion indicators $\scrE(\vec{x})$.
We can then write $\vec{\mdel} \sqsubseteq \vec{\mdel}'$ if $\vec{\mdel}'$ can be obtained from $\vec{\mdel}$ by 
deleting additional tokens.
This allows us to define a set of deletions building on $\vec{\mdel}$:
\begin{equation}
  \scrE(\vec{x}, \vec{\mdel}) \coloneq \{\, \vec{\mdel}' \in \scrE(\vec{x}) : \vec{\mdel} \sqsubseteq \vec{\mdel}' \,\}.
\end{equation}
The following result adapted from \citep[Lemma~4]{huang2023rsdel} identifies pairs of deletions $\pert{\medit}$ to 
$\pert{\vec{x}}$ and $\medit$ to $\vec{x}$ such that the terms $s(\pert{\vec{x}}, \pert{\medit})$ and 
$s(\vec{x}, \medit)$ are proportional.

\begin{lemma}[\citealp{huang2023rsdel}] \label{lem:equiv-edit-del}
  Let $\vec{z}^\star$ be a longest common subsequence~\citep{wagner1974string} of $\tokenizer(\pert{\vec{x}})$ and 
  $\tokenizer(\vec{x})$, and let 
  $\pert{\medit}^\star \in \scrE(\pert{\vec{x}})$ and $\medit^\star \in \scrE(\vec{x})$ be any deletions such that 
  $\apply(\pert{\vec{x}}, \pert{\medit}^\star) = \apply(\vec{x}, \medit^\star) = \tokenizer^{-1}(\vec{z}^\star)$.
  Then there exists a bijection $m\colon \scrE(\pert{\vec{x}}, \pert{\medit}^\star) \to \scrE(\vec{x}, \medit^\star)$ 
  such that $\apply(\pert{\vec{x}}, \pert{\medit}) = \apply(\pert{\vec{x}}, \medit^\star)$ for any 
  $\pert{\medit} \sqsupseteq \pert{\medit}^\star$.
  Furthermore for $\vec{\mdel} = m(\pert{\vec{\mdel}})$ we have
  \begin{equation*}
    s(\pert{\vec{\mdel}}, \pert{\vec{x}}) = 
      p_\del^{\abs{\tokenizer(\pert{\vec{x}})} - \abs{\tokenizer(\vec{x})}} s(\vec{\mdel}, \vec{x}).
  \end{equation*}
\end{lemma}

Using the above result, we can relate the smoothed classifier's score at $\pert{\vec{x}}$ and $\vec{x}$ as follows:
\begin{align}
  p_y(\pert{\vec{x}}) 
  &= \sum_{\pert{\vec{\mdel}} \in \scrE(\pert{\vec{x}})} s(\pert{\vec{\mdel}}, \pert{\vec{x}}) \nonumber \\
  & = p_\del^{\abs{\tokenizer(\pert{\vec{x}})} - \abs{\tokenizer(\vec{x})}} \left(
      p_y(\vec{x}) - \sum_{\vec{\mdel} \notin \scrE(\vec{x}, \vec{\mdel}^\star)} s(\vec{\mdel}, \vec{x}) 
    \right) \nonumber \\
  & \qquad {} + \sum_{\pert{\vec{\mdel}} \notin \scrE(\pert{\vec{x}}, \pert{\vec{\mdel}}^\star)} 
    s(\pert{\vec{\mdel}}, \pert{\vec{x}}). \label{eqn:relate-score-pair}
\end{align}
We can bound the sums in the above expression using the following result.
\begin{lemma}\label{lem:bound-remainder} 
  We have
  \begin{equation*}
    0 \leq \sum_{\vec{\mdel} \notin \scrE(\vec{x}, \vec{\mdel}^\star)} s(\vec{\mdel}, \vec{x}) \leq 1 - p_\del^{\sum_i \mdel_i^\star}.
  \end{equation*}
\end{lemma}
\begin{proof}
  The lower bound is straightforward, since the summand is non-negative.
  For the upper bound, observe that
  \begin{align*}
    & \sum_{\vec{\mdel} \notin \scrE(\vec{x}, \vec{\mdel}^\star)} s(\vec{\mdel}, \vec{x}; h) \\
    & = 1 - \sum_{\vec{\mdel} \in \scrE(\vec{x}, \vec{\mdel}^\star)} s(\vec{\mdel}, \vec{x}; h) \\
    & \leq 1 - \sum_{\vec{\mdel} \in \scrE(\vec{x}, \vec{\mdel}^\star)} q(\vec{\mdel}) \\
    & = 1 - p_\del^{\sum_i \mdel_i^\star} 
      \sum_{\vec{\mdel} \in \scrE(\vec{x}, \vec{\mdel}^\star)} q(\vec{\mdel} - \vec{\mdel}^\star) \\
    & \leq 1 - p_\del^{\sum_i \mdel_i^\star}
  \end{align*}
\end{proof}

Combining these results yields upper and lower bounds on the smoothed classifier's score for neighboring text 
$\pert{\vec{x}}$ to input text $\vec{x}$.

\pairwisecert*
\begin{proof}
  We obtain upper and lower bounds on the expression for $p_y(\pert{\vec{x}})$ in \eqref{eqn:relate-score-pair} 
  using Lemma~\ref{lem:bound-remainder}. 
  Replacing the sum over $\pert{\vec{\mdel}}$ by a lower bound and the sum over $\vec{\mdel}$ by an upper bound
  yields:
  \begin{equation*}
    p_y(\pert{\vec{x}}) \geq p_\del^{\abs{\tokenizer(\pert{\vec{x}})} - \abs{\tokenizer(\vec{x})}} \left(
      p_y(\vec{x}) - 1 + p_\del^{\sum_i \mdel_i^\star}
    \right).
  \end{equation*}
  Similarly, replacing the sum over $\pert{\vec{\mdel}}$ by an upper bound and the sum over $\vec{\mdel}$ by a lower 
  bound yields:
  \begin{equation*}
    p_y(\pert{\vec{x}}; h) \leq p_\del^{\abs{\tokenizer(\pert{\vec{x}})} - \abs{\tokenizer(\vec{x})}} 
    p_y(\vec{x}) + 1 - p_\del^{\sum_i \pert{\mdel}_i^\star} 
  \end{equation*}
  The final result is obtained by observing 
  $\abs{\tokenizer(\vec{x})} = \abs{\tokenizer(\pert{\vec{x}})} + n_\ins - n_\del$, 
  $\sum_i \pert{\mdel}_i^\star = n_\sub + n_\del$ and 
  $\sum_i \mdel_i^\star = n_\sub + n_\ins$.
\end{proof}

Finally we extend the pairwise certificate to a certificate over a Levenshtein (edit) distance ball.

\levcert*
\begin{proof}
  By definition, the smoothed classifier is robust in the Levenshtein distance neighborhood 
  $B_\radius(\vec{x}; \opset, \tokenizer)$ iff the difference between the score for the predicted class $y$ and any 
  other class $y'' \neq y$ is positive for all $\pert{\vec{x}} \in B_\radius(\vec{x}; \opset, \tokenizer)$:
  \begin{equation}
    \min_{\pert{\vec{x}} \in B_\radius(\vec{x}; \opset, \tokenizer)} \left\{ 
      p_y(\pert{\vec{x}}) - \max_{y'' \neq y} p_{y''}(\pert{\vec{x}}) 
      \right\} > 0.
    \label{eqn:tight-robust-defn}
  \end{equation}
  This condition is satisfied if a lower bound on the LHS is positive. 
  We obtain a lower bound on the LHS using bounds on $p_y(\pert{\vec{x}})$ from 
  Theorem~\ref{thm:del-cert-pairwise} and bounds on $p_y(\vec{x})$ from the theorem statement:
  \begin{equation*}
    \text{LHS of \eqref{eqn:tight-robust-defn}} \geq 
    \min_{\substack{n_\del, n_\ins, n_\sub \geq 0 \\\text{s.t.} n_\del + n_\ins + n_\sub \leq \radius}} 
      \psi(n_\del, n_\ins, n_\sub)
  \end{equation*}
  where the objective is
  \begin{equation*}
  \begin{split}
    & \psi(n_\del, n_\ins, n_\sub) = p_\del^{n_\del - n_\ins} (\mu_y - 1 + p_\del^{n_\sub + n_\ins}) \\
    & \qquad {} - p_\del^{n_\del - n_\ins} \mu_{y'} - 1 + p_\del^{n_\sub + n_\del}.
  \end{split}
  \end{equation*}
  It is straightforward to show that this objective is monotonically decreasing in $n_\sub$, hence the minimum occurs 
  at $(n_\del, n_\ins, n_\sub) = (0, 0, \radius)$. 
  Recalling that the lower bound $\psi(0, 0, \radius)$ must be positive to guarantee robustness, and solving for 
  $\radius$, yields 
  \begin{equation}
    \radius \leq \log_{p_\del}\left(\nicefrac{2 + \mu_{y'} - \mu_y}{2}\right). \label{eqn:del-cert-radius}    
  \end{equation}
  Enforcing the constraint that $\radius$ is an integer yields the required result.
\end{proof}

%% file: appendices/evaluation-parameters.tex
\section{Parameter settings}
\label{app-sec:parameters}

\begin{table*}[h]
  \begin{center}
    \small
    \begin{tabular}{cll}
      \toprule
                                           & \textbf{Parameter} & \textbf{Values}                                             \\\midrule\midrule
      \multirow{2}{*}{\textbf{Base model}} & Model              & \texttt{AutoModelForSequenceClassification("roberta-base")} \\\cmidrule(lr){2-3}
                                           & Tokenizer          & \texttt{AutoTokenizer("roberta-base")}                      \\\midrule
      \multirow{2}{*}{\textbf{Scheduler}}  & Python command     & \texttt{transformers.get\_linear\_schedule\_with\_warmup}   \\\cmidrule(lr){2-3}
                                           & Warmup epochs      & \texttt{10}                                                 \\\midrule
      \multirow{4}{*}{\textbf{Optimizer}}  & Python class       & \texttt{torch.optim.AdamW}                                  \\\cmidrule(lr){2-3}
                                           & Learning rate      & \texttt{2.0E-5}                                             \\\cmidrule(lr){2-3}
                                           & Weight decay       & \texttt{1.0E-6}                                             \\\cmidrule(lr){2-3}
                                           & Gradient clipping  & \texttt{clip\_grad\_norm\_(model.parameters(), 1.0)}        \\\midrule
      \multirow{4}{*}{\textbf{Training}}   & Batch size         & \texttt{32}                                                 \\\cmidrule(lr){2-3}
                                           & Max.\ epoch        & \texttt{200}                                                \\\cmidrule(lr){2-3}
                                           & Early stopping     & No improvement in validation loss after 25 epochs           \\\bottomrule
    \end{tabular}
  \end{center}
  \caption{
    Parameter settings for RoBERTa, the optimizer and training procedure.
    Parameter settings are consistent across all models (\ns, \randel, \ranmask) except where specified.
  }
  \label{tbl:model-parameters}
\end{table*}

We fine-tune the RoBERTa model for the non-certified baseline, \randel\ and \ranmask\ on the respective training
datasets.
The default parameter settings for the experiments are shown in Table~\ref{tbl:model-parameters}.
We do not explicitly calibrate the optimizer or training schedule for each model,
as we find the default settings work well across all datasets.
When approximating the smoothed models (\randel\ and \ranmask) we use a Monte Carlo sample of 1000 perturbed inputs
for prediction and 4000 perturbed inputs for certification, while setting the significance level to 0.05.

%% file: appendices/evaluation-certify.tex
\section{Certified robustness} \label{app-sec:certification-detailed}

\subsection{Certification statistics}
\label{app-sec:cr-statistics}

\begin{table}
  \begin{center}
    \small
    \begin{NiceTabular}{
        l@{}r
        r
        r
        r
      }[colortbl-like]
      \toprule
                                &                  & Clean            & Median & Median         \\
      Model                     & $p_\del/p_\mask$ & Accuracy         & CR     & $\log$ CC      \\
      \midrule
      \multicolumn{5}{c}{\rowcolor{gray!10} \agnews\ dataset (avg length 37.84)}                \\
      \midrule
      \ns                       & ---              & 94.84\%          & ---    & ---            \\
      \midrule
      \multirow{3}{*}{\ranmask} & 80\%             & \textbf{93.91\%} & 2      & 12.25          \\
                                & 90\%             & \textbf{92.43\%} & 4      & \textbf{22.91} \\
                                & 95\%             & \textbf{88.86\%} & 4      & \textbf{23.67} \\
      \midrule
      \multirow{3}{*}{\randel}  & 80\%             & 93.33\%          & 2      & \textbf{12.65} \\
                                & 90\%             & 91.72\%          & 3      & 18.65          \\
                                & 95\%             & 87.75\%          & 3      & 18.65          \\
      \midrule
      \multicolumn{5}{c}{\rowcolor{gray!10} \imdb\ dataset (avg length 231.16)}                 \\
      \midrule
      \ns                       & ---              & 93.47\%          & ---    & ---            \\
      \midrule
      \multirow{3}{*}{\ranmask} & 80\%             & \textbf{90.23\%} & 1      & 7.41           \\
                                & 90\%             & 86.87\%          & 2      & 14.02          \\
                                & 95\%             & 50.00\%          & 2      & 9.40           \\
      \midrule
      \multirow{3}{*}{\randel}  & 80\%             & 89.60\%          & 1      & \textbf{7.50}  \\
                                & 90\%             & \textbf{88.26\%} & 2      & \textbf{14.49} \\
                                & 95\%             & \textbf{85.58\%} & 3      & \textbf{21.20} \\
      \midrule
      \multicolumn{5}{c}{\rowcolor{gray!10} \assassin\ dataset (avg length 228.16)}             \\
      \midrule
      \ns                       & ---              & 98.02\%          & ---    & ---            \\
      \midrule
      \multirow{3}{*}{\ranmask} & 80\%             & \textbf{97.86\%} & 3      & 19.99          \\
                                & 90\%             & 97.65\%          & 6      & 37.49          \\
                                & 95\%             & 96.05\%          & 11     & 67.06          \\
      \midrule
      \multirow{3}{*}{\randel}  & 80\%             & 97.81\%          & 3      & \textbf{20.63} \\
                                & 90\%             & \textbf{97.81\%} & 6      & \textbf{38.66} \\
                                & 95\%             & \textbf{97.81\%} & 10     & 67.04          \\
      \midrule
      \multicolumn{5}{c}{\rowcolor{gray!10} \lun\ dataset (avg length 269.93)}                  \\
      \midrule
      \ns                       & ---              & 99.16\%          & ---    & ---            \\
      \midrule
      \multirow{3}{*}{\ranmask} & 80\%             & 98.67\%          & 3      & 19.73          \\
                                & 90\%             & 97.91\%          & 6      & 34.51          \\
                                & 95\%             & 95.62\%          & 10     & 60.45          \\
      \midrule
      \multirow{3}{*}{\randel}  & 80\%             & \textbf{98.85\%} & 3      & \textbf{20.62} \\
                                & 90\%             & \textbf{98.28\%} & 6      & \textbf{37.94} \\
                                & 95\%             & \textbf{96.11\%} & 10     & \textbf{61.44} \\
      \midrule
      \multicolumn{5}{c}{\rowcolor{gray!10} \satnews\ dataset (avg length 384.84)}              \\
      \midrule
      \ns                       & ---              & 94.22\%          & ---    & ---            \\
      \midrule
      \multirow{3}{*}{\ranmask} & 80\%             & 93.10\%          & 2      & 14.30          \\
                                & 90\%             & 92.09\%          & 4      & 27.84          \\
                                & 95\%             & 90.10\%          & 7      & 47.09          \\
      \midrule
      \multirow{3}{*}{\randel}  & 80\%             & \textbf{95.60\%} & 2      & \textbf{14.83} \\
                                & 90\%             & \textbf{93.18\%} & 5      & \textbf{35.08} \\
                                & 95\%             & \textbf{92.07\%} & 8      & \textbf{54.76} \\
      \bottomrule
    \end{NiceTabular}
  \end{center}
  \caption{
    Full certification results supplementing Table~\ref{tbl:cr-statistics-small}.
    All metrics are computed using the entire test set.
    ``Median CR'' is the median certified Levenshtein distance radius and ``median $\log$ CC'' is the median
    log-certified cardinality.
    The certified cardinality is exact for \ranmask, however a lower bound is used for \randel.
    \randel\ outperforms \ranmask\ in terms of certified accuracy for for 4 out of 5 datasets and 22 out of 30 metrics,
    and it specifically excels on datasets with longer average text length.
  }
  \label{tbl:cr-statistics}
\end{table}

We first present results of \randel\ and \ranmask\ on the \agnews, \lun, and \satnews\ datasets (Table~\ref{tbl:cr-statistics}).
In general, we see a minor drop in clean accuracy with increasing perturbation strength for both methods. 
Although in one case \ranmask\ suffers a catastrophic drop in accuracy to to 50\% for \imdb\ with a perturbation strength of 95\%. %
Surprisingly on \satnews, the smoothed classifiers achieve a higher clean accuracy than the baseline model.
While smoothing does not significantly impact accuracy on the \agnews, \lun\ and \satnews\ datasets,
it does have a pronounced impact on the \imdb\ dataset.
This is likely due to the fact that the \imdb\ sentiment classification task is more sensitive to small perturbations
such as names \citep{prabhakaran-etal-2019-perturbation}.

We observe that \randel\ dominates \ranmask\ in terms of clean accuracy and certified cardinality for 4 out of 5 datasets.
However, for the shorter \agnews\ dataset with an average input length of 37.84 words,
the results are more mixed, with \ranmask\ coming out on top when the perturbation strength is above $90\%$.
This is to be expected, as masking preserves more spatial information compared to deletion,
and it becomes advantageous when the input text sequence is shorter. We provide a more detailed analysis in Section~\ref{sec:eval-cr-input-size}.

\subsection{Certified accuracy plots}
\label{app-sec:ca-plots}

\begin{figure*}
  \centering
  \begin{subfigure}[b]{0.45\textwidth}
    \centering
    \includegraphics[width=\textwidth]{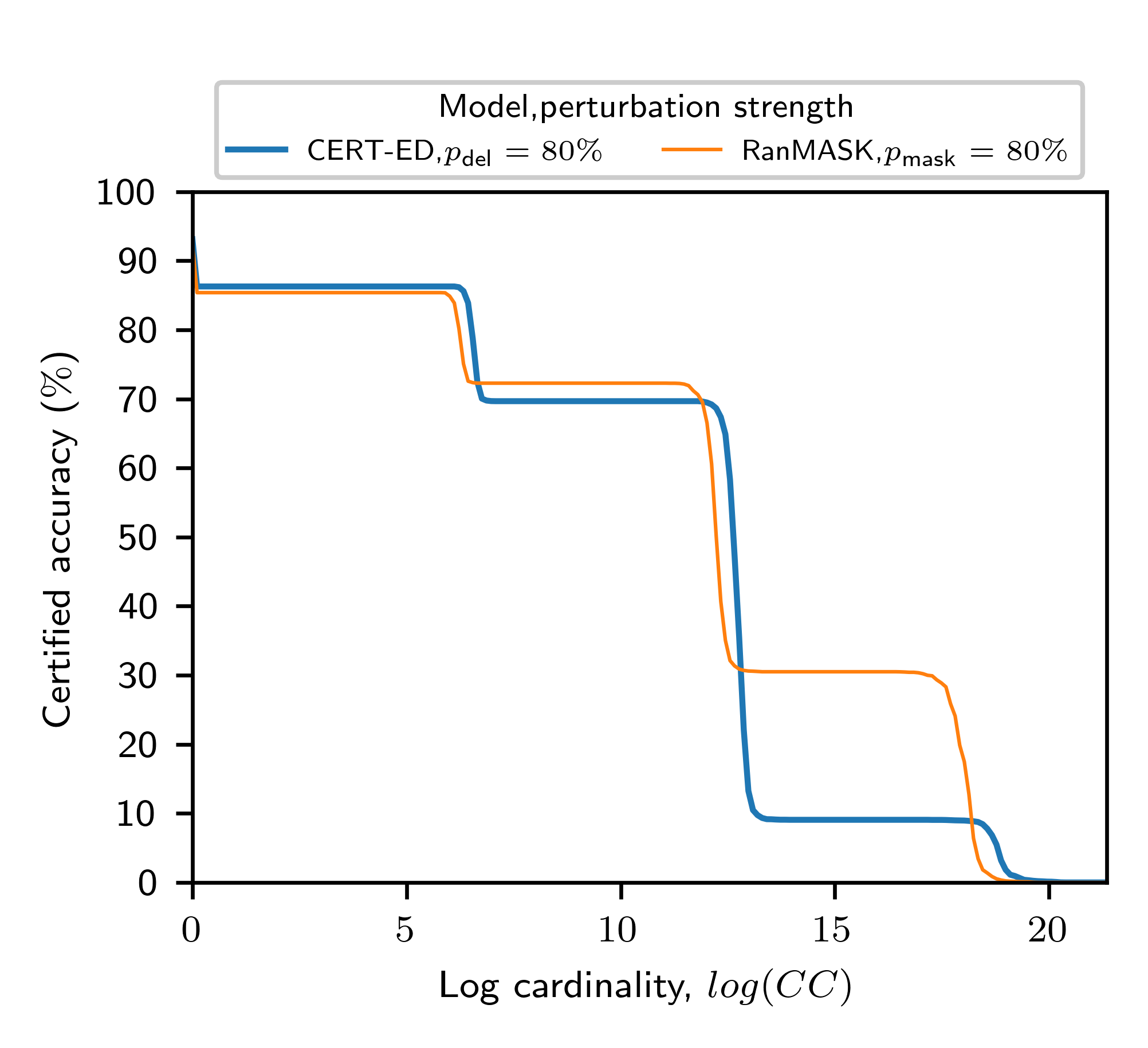}
    \caption{\agnews\ dataset}
    \label{fig:certified_accuracy-ag_news}
  \end{subfigure}
  \hfill
  \begin{subfigure}[b]{0.45\textwidth}
    \centering
    \includegraphics[width=\textwidth]{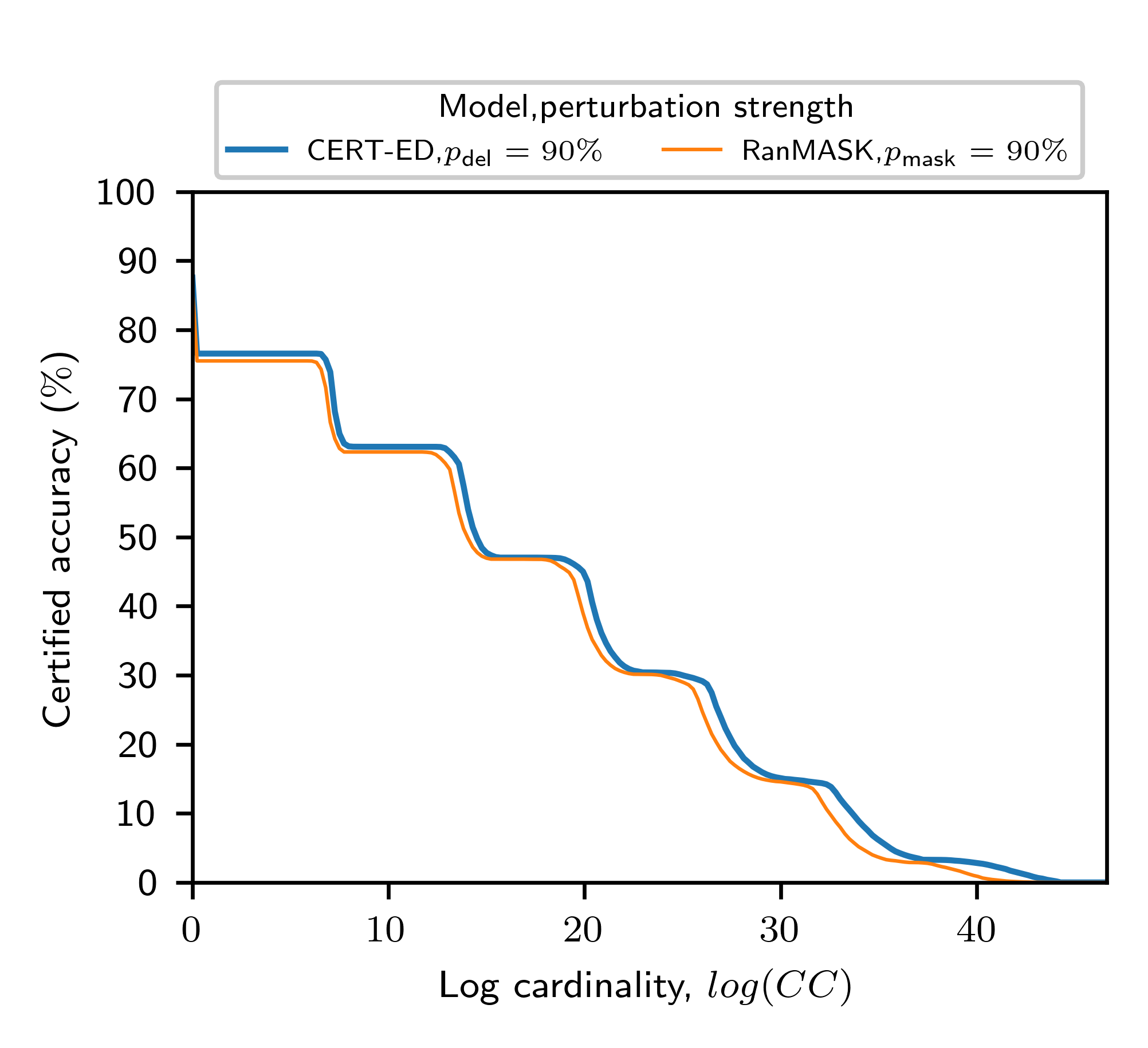}
    \caption{\imdb\ dataset}
    \label{fig:certified_accuracy-imdb}
  \end{subfigure}
  \vfill
  \begin{subfigure}[b]{0.45\textwidth}
    \centering
    \includegraphics[width=\textwidth]{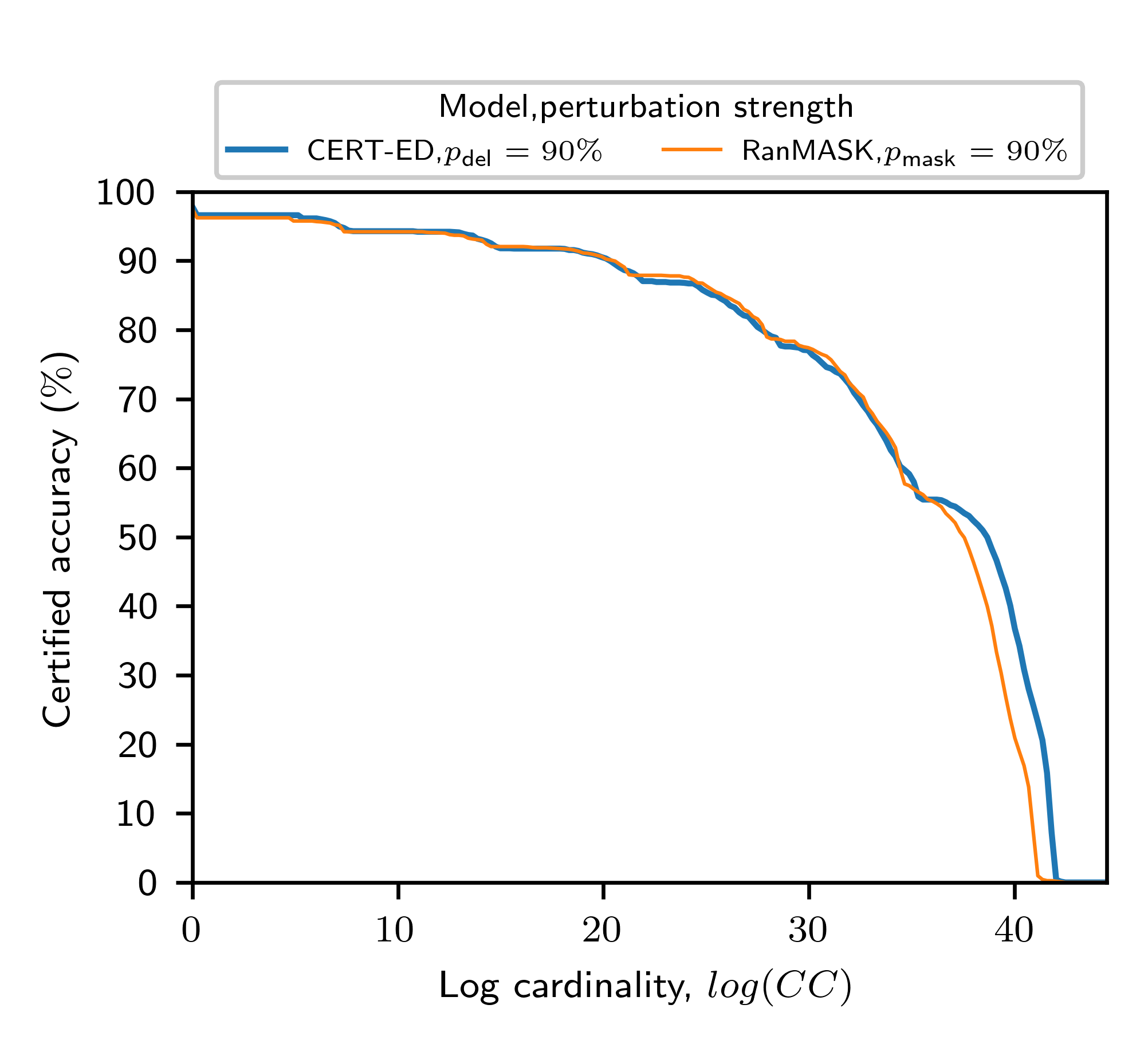}
    \caption{\assassin\ dataset}
    \label{fig:certified_accuracy-assassin}
  \end{subfigure}
  \hfill
  \begin{subfigure}[b]{0.45\textwidth}
    \centering
    \includegraphics[width=\textwidth]{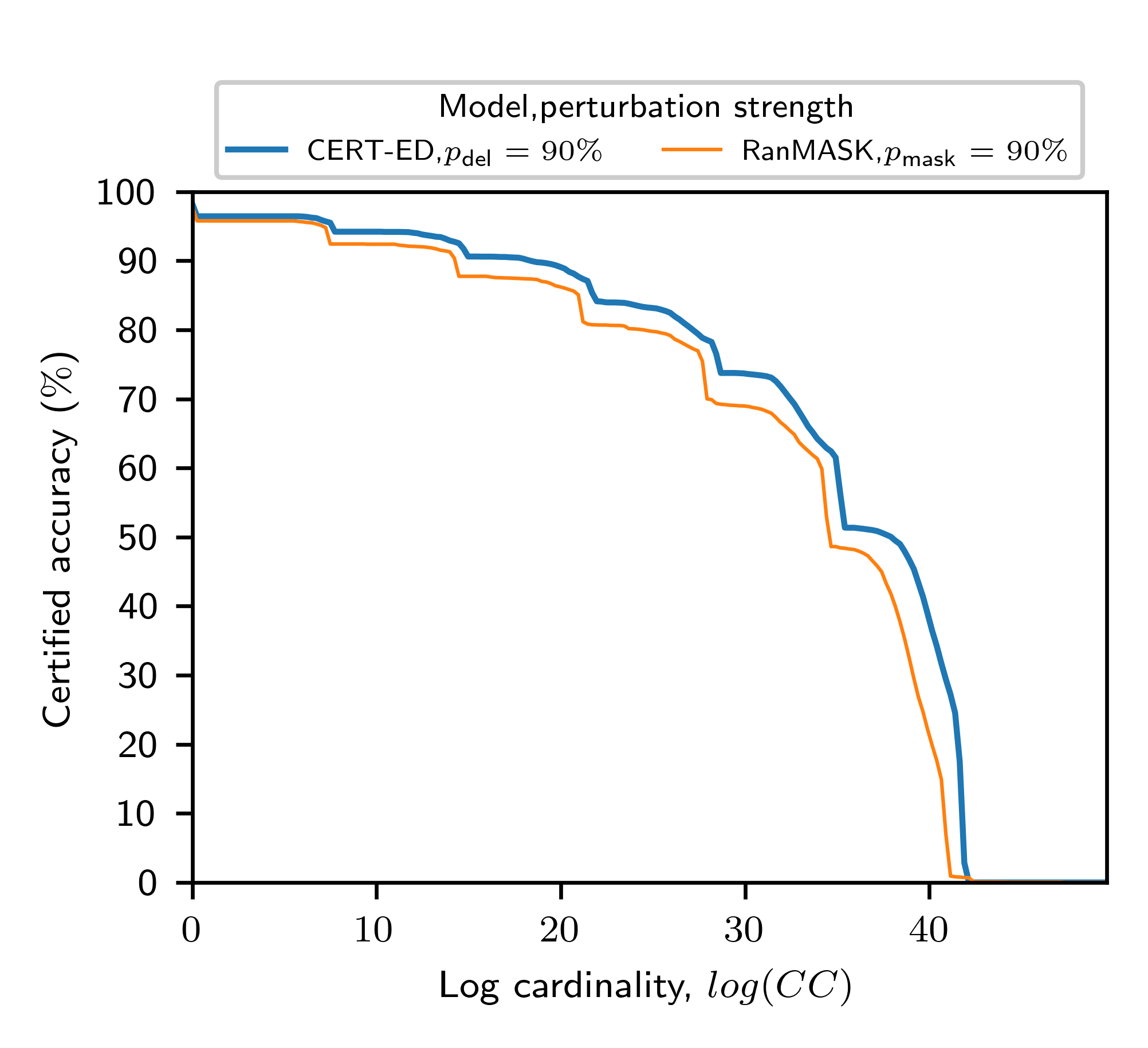}
    \caption{\lun\ dataset}
    \label{fig:certified_accuracy-lun}
  \end{subfigure}
  \caption{
    Certified accuracy for \randel\ and \ranmask\ as a function of log certified cardinality
    and perturbation strength $p_\del$ and $p_\mask$ (line styles).
    The certified cardinality is exact for \ranmask\ but a lower bound is used for \randel.
    \randel\ dominates \ranmask\ in terms of certified accuracy for 3 out of 4 datasets.
    See Figure~\ref{fig:certified_accuracy-satnews} for certified accuracy on \satnews. 
  }
  \label{fig:certified_accuracy-all}
\end{figure*}

Figure~\ref{fig:certified_accuracy-all} plots the certified accuracy as a function of the log certified cardinality
and perturbation strength for 4 datasets: \agnews, \imdb, \assassin, and \lun. 
The corresponding plot for \satnews\ is featured in Figure~\ref{fig:certified_accuracy-satnews}. 
Note that we are varying the cardinality of the certificate, rather than the radius, so that we can reasonably compare 
the size of \randel\ and \ranmask\ certificates, which are defined using different distance metrics 
(Levenshtein distance for \randel, and Hamming distance for \ranmask). 
We compute the cardinality exactly for Hamming distance and use a lower bound for Levenshtein distance, which typically 
underestimates the exact value by 1~order of magnitude.
Similar to the results in Table~\ref{tbl:cr-statistics-small},
\randel\ outperforms \ranmask\ on both \lun\ and \satnews\ for all perturbation strengths.
For \agnews, the results are more mixed with \ranmask\ outperforming \randel\ at a perturbation strength of $90\%$. 

\subsection{Impact of deletion rate}
\begin{table}
  \begin{center}
    \small
    \begin{NiceTabular}{
        l
        r
        r
        r
      }[colortbl-like]
      \toprule
      \randel  & Clean    & Median & Median                                         \\
      $p_\del$ & Accuracy & CR     & $\log$ CC                                      \\
      \midrule
      \multicolumn{4}{c}{\rowcolor{gray!10} \agnews\ dataset (avg length 37.84)}    \\
      \midrule
      50\%     & 94.76    & 0      & 0.00                                           \\
      60\%     & 95.07    & 1      & 6.55                                           \\
      70\%     & 94.58    & 1      & 6.56                                           \\
      80\%     & 93.33    & 2      & 12.65                                          \\
      90\%     & 91.72    & 3      & 18.65                                          \\
      95\%     & 87.75    & 3      & 18.65                                          \\
      99\%     & 25.05    & 0      & 0.00                                           \\
      \midrule
      \multicolumn{4}{c}{\rowcolor{gray!10} \imdb\ dataset (avg length 231.16)}     \\
      \midrule
      50\%     & 94.34    & 0      & 0.00                                           \\
      60\%     & 93.16    & 1      & 7.11                                           \\
      70\%     & 93.00    & 1      & 7.14                                           \\
      80\%     & 89.60    & 1      & 7.50                                           \\
      90\%     & 88.26    & 2      & 14.49                                          \\
      95\%     & 85.58    & 3      & 21.20                                          \\
      99\%     & 68.02    & 6      & 39.64                                          \\
      \midrule
      \multicolumn{4}{c}{\rowcolor{gray!10} \assassin\ dataset (avg length 228.16)} \\
      \midrule
      50\%     & 98.19    & 0      & 0.00                                           \\
      60\%     & 98.49    & 1      & 7.32                                           \\
      70\%     & 98.23    & 1      & 7.32                                           \\
      80\%     & 97.81    & 3      & 20.63                                          \\
      90\%     & 97.81    & 6      & 38.66                                          \\
      95\%     & 97.81    & 10     & 67.04                                          \\
      99\%     & 84.19    & 25     & 152.84                                         \\
      \midrule
      \multicolumn{4}{c}{\rowcolor{gray!10} \lun\ dataset (avg length 269.93)}      \\
      \midrule
      50\%     & 99.29    & 0      & 0.00                                           \\
      60\%     & 99.33    & 1      & 7.48                                           \\
      70\%     & 99.40    & 1      & 7.48                                           \\
      80\%     & 98.85    & 3      & 20.62                                          \\
      90\%     & 98.28    & 6      & 37.94                                          \\
      95\%     & 96.11    & 10     & 61.44                                          \\
      99\%     & 86.89    & 19     & 118.78                                         \\
      \midrule
      \multicolumn{4}{c}{\rowcolor{gray!10} \satnews\ dataset (avg length 384.84)}  \\
      \midrule
      50\%     & 94.97    & 0      & 0.00                                           \\
      60\%     & 96.53    & 1      & 7.59                                           \\
      70\%     & 95.03    & 1      & 7.59                                           \\
      80\%     & 95.60    & 2      & 14.83                                          \\
      90\%     & 93.18    & 5      & 35.08                                          \\
      95\%     & 92.07    & 8      & 54.76                                          \\
      99\%     & 85.81    & 16     & 104.73                                         \\
      \bottomrule
    \end{NiceTabular}
  \end{center}
  \caption{
    Ablation study of \randel\ with varying deletion rates $p_\del$.
    All metrics are computed using the entire test set.
    ``Median CR'' is the median certified Levenshtein distance radius and ``median $\log$ CC'' is the median
    log-certified cardinality.
    The certified cardinality is estimated using a lower bound.
    For datasets with longer text, the deletion rate has less impact on accuracy.
  }
  \label{tbl:cr-ablation}
\end{table}

We present a comprehensive ablation study of \randel\ with varying deletion rates in Table~\ref{tbl:cr-ablation}.
We observe that for shorter sequences, the deletion rate has a more significant impact on the clean accuracy.
For \imdb, the maximum possible certified cardinality is much lower compared to other datasets,
further demonstrating the sensitivity of sentiment classification to perturbations.
Being the simplest dataset, \assassin\ has the highest certified cardinality among all datasets.

\subsection{Impact of input text length}
\label{sec:eval-cr-input-size}
\begin{table*}
  \begin{center}
    \small
    \begin{NiceTabular}{
      c@{\hspace{5px}}c
      |c@{\hspace{5px}}c
      |c@{\hspace{5px}}c
      |c@{\hspace{5px}}c
      |c@{\hspace{5px}}c
      |c@{\hspace{5px}}c
      }[colortbl-like]
      \toprule
         &
         & \multicolumn{2}{c}{\ns}
         & \multicolumn{2}{c}{\ranmask\ $80\%$}
         & \multicolumn{2}{c}{\ranmask\ $90\%$}
         & \multicolumn{2}{c}{\randel\ $80\%$}
         & \multicolumn{2}{c}{\randel\ $90\%$}                                                                                             \\
      \cmidrule(lr){3-4} \cmidrule(lr){5-6} \cmidrule(lr){7-8} \cmidrule(lr){9-10} \cmidrule(lr){11-12}
      \textbf{Quartile}
         & \textbf{Avg. length}
         & \textbf{ClA\%}                       & \textbf{$\log$ CC}
         & \textbf{ClA\%}                       & \textbf{$\log$ CC}
         & \textbf{ClA\%}                       & \textbf{$\log$ CC}
         & \textbf{ClA\%}                       & \textbf{$\log$ CC}
         & \textbf{ClA\%}                       & \textbf{$\log$ CC}
      \\
      \midrule
      \multicolumn{12}{c}{\rowcolor{gray!10} \agnews\ dataset}                                                                             \\
      \midrule
      Q1 & 27.25                                & 93.62              & --- & 92.33 & 11.95 & 90.85 & 17.71 & 91.22 & 12.39 & 89.37 & 12.60 \\
      Q2 & 35.03                                & 94.93              & --- & 93.90 & 12.18 & 92.99 & 17.99 & 93.90 & 12.65 & 92.54 & 18.58 \\
      Q3 & 40.38                                & 95.21              & --- & 94.00 & 12.32 & 92.33 & 23.72 & 93.74 & 12.79 & 92.43 & 18.81 \\
      Q4 & 50.74                                & 95.88              & --- & 95.82 & 12.49 & 94.00 & 24.17 & 94.94 & 12.93 & 93.06 & 24.72 \\
      \midrule
      \multicolumn{12}{c}{\rowcolor{gray!10} \imdb\ dataset}                                                                               \\
      \midrule
      Q1 & 92.32                                & 94.74              & --- & 92.12 & 6.79  & 89.56 & 13.18 & 90.76 & 7.05  & 88.94 & 13.48 \\
      Q2 & 147.20                               & 95.46              & --- & 92.56 & 13.31 & 89.06 & 13.55 & 90.51 & 7.21  & 89.20 & 13.98 \\
      Q3 & 215.62                               & 94.17              & --- & 90.91 & 13.59 & 86.88 & 20.05 & 89.30 & 7.41  & 87.61 & 14.40 \\
      Q4 & 461.22                               & 89.51              & --- & 85.31 & 7.40  & 81.93 & 14.42 & 87.83 & 14.51 & 87.28 & 21.61 \\
      \midrule
      \multicolumn{12}{c}{\rowcolor{gray!10} \assassin\ dataset}                                                                           \\
      \midrule
      Q1 & 74.23                                & 98.01              & --- & 97.51 & 18.59 & 97.84 & 31.80 & 98.18 & 13.78 & 97.68 & 30.93 \\
      Q2 & 171.62                               & 99.32              & --- & 99.49 & 19.97 & 99.49 & 38.35 & 99.15 & 20.54 & 99.49 & 39.10 \\
      Q3 & 281.63                               & 96.80              & --- & 96.80 & 20.55 & 96.13 & 39.50 & 96.63 & 21.14 & 97.14 & 40.51 \\
      Q4 & 395.60                               & 97.97              & --- & 97.64 & 21.07 & 97.13 & 40.72 & 97.30 & 21.70 & 96.96 & 41.70 \\
      \midrule
      \multicolumn{12}{c}{\rowcolor{gray!10} \lun\ dataset}                                                                                \\
      \midrule
      Q1 & 84.65                                & 99.26              & --- & 98.64 & 18.62 & 98.21 & 35.79 & 98.70 & 16.87 & 97.16 & 32.02 \\
      Q2 & 226.05                               & 99.57              & --- & 99.57 & 20.24 & 99.19 & 39.07 & 99.32 & 20.83 & 99.50 & 39.88 \\
      Q3 & 363.68                               & 98.32              & --- & 97.58 & 14.26 & 96.40 & 34.24 & 98.07 & 21.55 & 97.20 & 41.21 \\
      Q4 & 404.00                               & 99.50              & --- & 98.88 & 21.09 & 97.83 & 34.43 & 99.32 & 21.72 & 99.26 & 41.75 \\
      \midrule
      \multicolumn{12}{c}{\rowcolor{gray!10} \satnews\ dataset}                                                                            \\
      \midrule
      Q1 & 353.34                               & 95.23              & --- & 96.12 & 14.24 & 94.51 & 27.71 & 97.06 & 14.73 & 95.23 & 34.93 \\
      Q2 & 384.64                               & 94.97              & --- & 93.76 & 14.28 & 92.87 & 27.77 & 95.69 & 14.76 & 93.98 & 35.11 \\
      Q3 & 398.19                               & 93.32              & --- & 92.05 & 14.30 & 91.20 & 27.82 & 95.60 & 14.79 & 92.10 & 35.19 \\
      Q4 & 416.96                               & 93.36              & --- & 90.36 & 14.34 & 89.66 & 27.91 & 93.95 & 14.86 & 91.36 & 35.27 \\
      \bottomrule
    \end{NiceTabular}
  \end{center}
  \caption{
    Clean accuracy (ClA) and median log-certified cardinality (log CC) for \ranmask\ and \randel\ as a function of perturbation strength grouped
    by input text length quintiles.
    The entire test set is used to compute all metrics.
    Q1 refers to the first quintile while Q4 refers to the fourth quintile.
    The certified cardinality is exact for \ranmask\ but a lower bound is used for \randel.
  }
  \label{tbl:cr-input-size}
\end{table*}
The results of clean accuracy and certified cardinality, grouped by text length quintiles, are shown in Table~\ref{tbl:cr-input-size}.
For the \agnews\ dataset, which has a shorter average text length,
\ranmask\ suffers a minor drop in clean accuracy from Q1 to Q2,
while \randel\ is impacted more significantly.
This supports our hypothesis that \ranmask\ is more advantageous when the input text sequence is shorter.

%% file: appendices/evaluation-attack.tex
\section{Details on attacks}
\label{app-sec:eval-attack}

\subsection{Attack setup}
\paragraph{Attacks covered}
We evaluate empirical robustness under various attacks using a modified version of \texttt{TextAttack}~\citep{morris2020textattack}
and attack templates implemented by \citet{zhang2024textcrs}.
We select five representative attacks which can be categorized as follows:
\begin{itemize}[leftmargin=*]
  \item \clare~\citep{li-etal-2021-contextualized} and \bae~\citep{garg-ramakrishnan-2020-bae} both cover a token-wise edit distance threat model;
  \item \bertattack~\citep{li-etal-2020-bert-attack} and \textfooler~\citep{jin2020bert} operate by substituting words in the input text;
  \item \deepwordbug~\citep{gao2018black} modifies the input text by altering characters within each word.
\end{itemize}

\paragraph{Attack results}
Each attack can yield one of four distinct outcomes, namely, \textit{success}, \textit{fail}, \textit{skipped} or \textit{timeout},
Their meanings are as follows:
\begin{itemize}[leftmargin=*]
  \item \textit{success} indicates the attack was able to generate an adversarial example by perturbing the prediction from the correct label to a false label.
  \item \textit{fail} indicates the attack was unable to generate an adversarial example.
        This can happen for the following reasons:
        firstly, the attack was unable to find a perturbation that changes the prediction;
        secondly, the attack was unable to find a perturbation that changes the prediction after exhausting all options;
        or, lastly, the attack reached the maximum number of queries on the target model. 
        We enforce a maximum limit of 10000 queries for Clare, but place no limit on the other attacks.
  \item \textit{skipped} indicates the attack was skipped because the model's prediction was incorrect in the first place.
  \item \textit{timeout} indicates the attack was skipped because it took too long to generate an adversarial example.
        In our experiments, we set the timeout to be 600 seconds. Note that this generally puts \randel\ at a disadvantage compared to 
        \ranmask, because \randel\ can process queries roughly 3~times faster than \ranmask.
        Unless otherwise specified, we treat timeout as fail.
\end{itemize}

The robust accuracy is defined as the fraction of instances for which the attack outcome is either \textit{fail} or \textit{timeout}---i.e., 
the fraction of instances for which the model's prediction remains correct after the attack.

\subsection{Transfer attack on \texttt{ag-news}}

\begin{table*}[h]
  \begin{center}
    \small
    \begin{NiceTabular}{
      l
      |c@{\hspace{5px}}c
      |c@{\hspace{5px}}c
      |c@{\hspace{5px}}c
      |c@{\hspace{5px}}c
      |c@{\hspace{5px}}c
      }[colortbl-like]
      \toprule
               & \multicolumn{2}{c}{\textbf{\clare}}
               & \multicolumn{2}{c}{\textbf{\bae}}
               & \multicolumn{2}{c}{\textbf{\bertattack}}
               & \multicolumn{2}{c}{\textbf{\textfooler}}
               & \multicolumn{2}{c}{\textbf{\deepwordbug}}                                                                                                                                                          \\
      \cmidrule(lr){2-3} \cmidrule(lr){4-5} \cmidrule(lr){6-7} \cmidrule(lr){8-9} \cmidrule(lr){10-11}
      \textbf{Method}
               & \textbf{ClA\%}                            & \textbf{RoA\%}                                                                                                                                         %
               & \textbf{ClA\%}                            & \textbf{RoA\%}                                                                                                                                         %
               & \textbf{ClA\%}                            & \textbf{RoA\%}                                                                                                                                         %
               & \textbf{ClA\%}                            & \textbf{RoA\%}                                                                                                                                         %
               & \textbf{ClA\%}                            & \textbf{RoA\%}                                                                                                                                         %
      \\
      \midrule
      \multicolumn{11}{c}{\rowcolor{gray!10} \agnews\ dataset}                                                                                                                                                      \\
      \midrule
      \ranmask & 93.51                                     & \textbf{90.41} & 89.50          & 86.46          & 93.30          & 88.28          & 94.08          & \textbf{90.77} & 91.41          & \textbf{82.81} \\
      \randel  & \textbf{93.66}                            & 90.27          & \textbf{90.06} & \textbf{86.74} & \textbf{93.46} & \textbf{89.04} & \textbf{94.21} & 89.67          & \textbf{91.60} & 81.45          \\
      \midrule
      \multicolumn{11}{c}{\rowcolor{gray!10} \imdb\ dataset}                                                                                                                                                        \\
      \midrule
      \ranmask & 79.33                                     & \textbf{75.51} & 85.31          & 82.31          & 87.73          & 81.66          & 87.67          & 77.30          & 84.26          & 76.54          \\
      \randel  & \textbf{82.02}                            & 75.28          & \textbf{86.81} & \textbf{82.76} & \textbf{89.02} & \textbf{83.06} & \textbf{88.94} & \textbf{78.57} & \textbf{85.96} & \textbf{76.70} \\
      \midrule
      \multicolumn{11}{c}{\rowcolor{gray!10} \assassin\ dataset}                                                                                                                                                    \\
      \midrule
      \ranmask & \textbf{94.92}                            & \textbf{91.53} & \textbf{95.74} & \textbf{91.49} & 98.49          & 95.27          & 98.21          & \textbf{94.47} & 96.57          & 88.57          \\
      \randel  & \textbf{94.92}                            & 89.83          & \textbf{95.74} & 89.36          & \textbf{98.71} & \textbf{96.34} & \textbf{98.70} & 94.31          & \textbf{98.86} & \textbf{93.14} \\
      \midrule
      \multicolumn{11}{c}{\rowcolor{gray!10} \lun\ dataset}                                                                                                                                                         \\
      \midrule
      \ranmask & 92.16                                     & 90.20          & 95.29          & 93.72          & 97.14          & 88.57          & 96.53          & 87.86          & 94.20          & 86.47          \\
      \randel  & \textbf{96.08}                            & \textbf{94.12} & \textbf{96.86} & \textbf{96.86} & \textbf{98.33} & \textbf{94.76} & \textbf{98.27} & \textbf{92.49} & \textbf{96.62} & \textbf{91.79} \\
      \midrule
      \multicolumn{11}{c}{\rowcolor{gray!10} \satnews\ dataset}                                                                                                                                                     \\
      \midrule
      \ranmask & 72.16                                     & 69.07          & 86.15          & 83.38          & 88.48          & 82.95          & 87.95          & 82.05          & 85.59          & 79.88          \\
      \randel  & \textbf{80.41}                            & \textbf{77.32} & \textbf{89.20} & \textbf{86.98} & \textbf{91.24} & \textbf{87.33} & \textbf{90.45} & \textbf{87.50} & \textbf{88.89} & \textbf{86.19} \\
      \bottomrule
    \end{NiceTabular}
  \end{center}
  \caption{
    Empirical attack results when transferring \emph{successful} adversarial examples against the non-certified baseline to 
    \randel\ and \ranmask.
    Both \randel\ and \ranmask\ use a perturbation strength of 90\%.
    Clean and robust accuracy are abbreviated ClA and RoA, respectively.
    Highlighted values are the best in each column for that dataset.
    \randel\ outperforms \ranmask\ in all word substitution and character-level attacks.
  }
  \label{tbl:transfer-attack}
\end{table*}

We perform transfer attacks by applying successful attack examples against the non-certified baseline to the smoothed models (\randel, \ranmask).
Table~\ref{tbl:transfer-attack} reports both clean and robust accuracy since the successful example against each attack will be slightly different.
Unlike the direct attack results, \randel\ does not have a strict dominance over \ranmask\ against \bertattack, \textfooler, and \deepwordbug.
While for \clare\ and \bae, \randel\ is marginally better rather than tied.
These results demonstrate the robustness of \randel\ to transfer attacks,
showing that \randel\ is more robust to adversarial examples for the non-certified baseline than \ranmask.

%% file: appendices/evaluation-efficiency.tex
\section{Efficiency and computation requirements}
\label{app-sec:eval-efficiency}
In this appendix, we document the computation requirements to train, certify, and attack models used in our work.
We also compare and contrast the efficiency of \randel\ and \ranmask\ in terms of training and certification.
We show that \randel\ is more efficient than \ranmask\ in both aspects.

\subsection{Hardware}
All experiments in this paper are conducted using a private cluster with Intel(R) Xeon(R) Gold 6326 CPU @ 2.90GHz and NVIDIA A100 GPUs.
Unless otherwise specified, we use a single GPU for all experiments.

\subsection{Train}
\begin{table*}
  \begin{center}
    \small
    \begin{NiceTabular}{l@{}r|rrrrrr}
      \toprule
                & Train     & \multicolumn{2}{c}{\ns} & \multicolumn{2}{c}{\ranmask, 90\%} & \multicolumn{2}{c}{\randel, 90\%}                                  \\
      Dataset   & \#samples & epochs                  & sec/epoch                          & epochs                            & sec/epoch & epochs & sec/epoch \\
      \midrule
      \agnews   & 108\,000  & 65                      & 517                                & 105                               & 476       & 100    & 231       \\
      \imdb     & 22\,500   & 30                      & 258                                & 60                                & 341       & 65     & 128       \\
      \assassin & 2\,152    & 40                      & 27                                 & 50                                & 35        & 40     & 13        \\
      \lun      & 13\,416   & 55                      & 143                                & 65                                & 258       & 60     & 55        \\
      \satnews  & 22\,738   & 55                      & 260                                & 80                                & 461       & 95     & 101       \\
      \bottomrule
    \end{NiceTabular}
  \end{center}
  \caption{
    Training time statistics for each dataset and model. 
    The number of epochs varies due to early stopping.
  }
  \label{tbl:train-time}
\end{table*}

Table~\ref{tbl:train-time} shows the number of epochs used to train each model\slash dataset (with early stopping) and the training time per epoch.
\randel\ is about 2--3~times faster to train than the non-smoothed baseline, and 2--5~times faster to train than \ranmask.
The total computation used across all datasets for certification is estimated to be 70~hours A100~GPU time.

\subsection{Certification}
\begin{table*}
  \begin{center}
    \small
    \begin{NiceTabular}{l@{}r|rrrr}
      \toprule
                & Test      & \multicolumn{1}{c}{\ranmask, 90\%} & \multicolumn{1}{c}{\randel, 90\%} \\
      Dataset   & \#samples & \multicolumn{1}{c}{ms / sample}    & \multicolumn{1}{c}{ms/sample}     \\
      \midrule
      \agnews   & 7\,600    & 3\,969                             & 2\,367                            \\
      \imdb     & 25\,000   & 13\,331                            & 3\,311                            \\
      \assassin & 2\,378    & 13\,899                            & 3\,319                            \\
      \lun      & 6\,454    & 14\,641                            & 4\,819                            \\
      \satnews  & 7\,202    & 17\,767                            & 5\,778                            \\
      \bottomrule
    \end{NiceTabular}
  \end{center}
  \caption{
    Certification time on the test set for each dataset, including overheads.
    We use 1\,000 Monte Carlo samples for prediction and 4\,000 samples for estimating certified radii.
    During attacks, we use 100~samples for prediction, which cuts the prediction time by 1/40 ignoring overheads.
  }
  \label{tbl:certify-time}
\end{table*}

Table~\ref{tbl:certify-time} shows the average certification time per test instance, including overheads.
We see \randel\ is about 3~times faster than \ranmask\ on average across all datasets.
The total computation used across all datasets for certification is estimated to be 250~hours A100~GPU time.

\subsection{Empirical robustness}

\begin{table*}
  \begin{center}
    \small
    \begin{NiceTabular}{l|rrr}
      \toprule
                   & \multicolumn{1}{c}{\ns} & \multicolumn{1}{c}{\ranmask, 90\%} & \multicolumn{1}{c}{\randel, 90\%} \\
      Dataset      & sec / sample            & sec / sample                       & sec / sample                      \\
      \midrule
      \clare       & 192                     & 527                                & 504                               \\
      \bae         & 268                     & 502                                & 461                               \\
      \bertattack  & 34                      & 333                                & 234                               \\
      \textfooler  & 12                      & 302                                & 173                               \\
      \deepwordbug & 7                       & 155                                & 62                                \\
      \bottomrule
    \end{NiceTabular}
  \end{center}
  \caption{Attack time per instance on a subset of 1000 instances from the \imdb\ test set.
    The timeout window is set to be 10~minutes. 
    Note that most \clare\ and \bae\ attacks targeting \ranmask\ timed out.
    This partially explains the lower robust accuracy of \randel\ compared to \ranmask\ in Table~\ref{tbl:direct-attack}.
  }
  \label{tbl:attack-sample-time}
\end{table*}

Table~\ref{tbl:attack-sample-time} reports the average attack time per instance for a subset of the \imdb\ test set.
Attack times on other datasets follow a similar pattern.
We note that the time taken to attack \ranmask\ is longer than for the non-smoothed baseline and \randel.
Combined with the max timeout window of 10~minutes, this partially explains the lower robust accuracy of
\randel\ compared to \ranmask\ in Table~\ref{tbl:direct-attack}.
We utilized parallelized attacks to speed up the process.
The total computation used across all datasets for attacks is estimated to be 200~days A100~GPU time.

%% file: appendices/textcrs.tex
\section{Edit distance certificates for Text-CRS} \label{app-sec:textcrs}

In this appendix, we analyze two robustness certificates from the Text-CRS 
framework~\citep{zhang2024textcrs}, which cover deletion\slash insertion 
perturbations to text represented as a sequence of token embedding vectors. 
Each certificate is parameterized by two radii: one bounds the perturbation 
to the token embedding vectors and the other bounds the extent of token-level 
reordering. 
We obtain lower bounds on these two radii such that each certificate covers 
up to $\radius$ \emph{arbitrary} edits of a single type (deletion\slash 
insertion).
Using these bounds, we can immediately convert a Text-CRS deletion\slash 
insertion certificate to an edit distance certificate where the allowed edit 
operations are deletions\slash insertions to input tokens. 
We find that the resulting edit distance certificates are vacuous 
($\radius = 0$) for sequences greater than 2 tokens in length when instantiated 
with the Text-CRS smoothing mechanisms. 
As a result, we have opted not to include Text-CRS as a baseline in our 
experiments. 

\subsection{Preliminaries}

Text-CRS indirectly bounds edits to input text by instead bounding 
numerical additive perturbations and permutations in word embedding space. 
This is not in one-to-one correspondence as we see input edits are not easily 
bounded. 
Concretely, input text is represented as an array of embedding vectors 
$\vec{w} \in \reals^{n \times k}$, 
where the first dimension corresponds to words and the second dimension 
corresponds to dimensions in the embedding space.
Input text with $m < n$ actual words is represented by filling the last 
$n - m$ rows with padding words\slash vectors.  
Robustness is studied under input transformations that are a composition of:
(1)~perturbations to the embedding vectors and (2)~word-level permutations 
of the embedding vectors. 
Here we define notation to represent these transformations. 

\paragraph{Embedding perturbations}
An embedding perturbation is an array $\vec{\delta}$ of the same type as the 
input $\vec{w}$. 
The result of applying $\vec{\delta}$ to $\vec{w}$ is simply 
$\vec{w} + \vec{\delta}$.
We consider two norms to measure the magnitude of the perturbation:
\begin{itemize}[leftmargin=*]
  \item $\| \vec{\delta} \|_0 \coloneqq \sum_{i = 1}^{n} \ind{\sum_{j = 1}^{n} \abs{\delta_{i,j}} \neq 0}$ 
  is the sum of non-zero rows in $\vec{\delta}$, i.e., the number of perturbed 
  words\slash vectors; and
  \item $\| \vec{\delta} \|_2 \coloneqq \sqrt{\sum_{i = 1}^{n} \delta_{i,j}^2}$ 
  is the Frobenius norm.
\end{itemize}

\paragraph{Permutations}
A word-level permutation of an input $\vec{w}$ is parameterized by a 
permutation matrix $\vec{\pi} \in \mathcal{P}_n$.
Here $\mathcal{P}_n = \{ \vec{\pi} \in \{0, 1\}^{n \times n} : 
\sum_{i' = 1}^{n} \pi_{i',j} = 1, \sum_{j' = 1}^{n} \pi_{i,j'} = 1 
\forall i, j \}$ denotes the set of $n \times n$ permutation matrices. 
The result of applying $\vec{\pi}$ to $\vec{w}$ is simply $\pi \cdot \vec{w}$ 
where $\cdot$ denotes matrix multiplication.
The magnitude of the perturbation $\| \vec{\pi} \|_1$ is measured in terms of 
the $\ell_1$ distance between the new word locations and the original 
locations.
We can equivalently express this as the row-wise sum of the absolute distance 
of each 1 from the diagonal in the permutation matrix $\vec{\pi}$:
$\| \vec{\pi} \|_1 = \sum_{i = 1}^{n} \abs{i - \sum_{j = 1}^{n} j \ind{\pi_{i, j} = 1}}$.

\paragraph{Composition}
The composition of the embedding perturbation $\vec{\delta}$ and 
permutation $\vec{\pi}$ is an input transformation 
$T_{\vec{\delta}, \vec{\pi}}: \reals^{n \times d} \to \reals^{n \times d}$ such 
that
$T_{\vec{\delta}, \vec{\pi}}(\vec{w}) = \vec{\pi} \cdot (\vec{w} + \vec{\delta})$.

\subsection{Word-level deletion}

Text-CRS covers deletion using a certificate that constrains the 
number of modified embedding vectors and the sum of word position changes.

\begin{definition}
  A Text-CRS deletion certificate at input $\vec{w} \in \reals^{n \times d}$ 
  is a set of inputs parameterized by two radii $\radius_R, \radius_D \geq 0$:
  \begin{align*}
    & C_D(\vec{w}; \radius_R, \radius_D) = \{ \, 
      T_{\vec{\delta}, \vec{\pi}}(\vec{w}) \in \reals^{n \times d} : 
        \vec{\pi} \in \mathcal{P}_n, \\
      & \quad \vec{\delta} \in \reals^{n \times d}, \|\vec{\pi}\|_1 < \radius_R, 
      \|\vec{\delta}\|_0 < \radius_D
    \, \}.
  \end{align*}  
\end{definition}

We note that this form of certificate is not tight for deletion. 
For example, it includes invalid inputs that contain padding in the middle of 
the sequence, and it also includes inputs where words are replaced by 
ordinary (non-padding) words. 

We are interested in determining values of $\radius_R$ and $\radius_D$ such 
that the Text-CRS deletion certificate covers a standard edit distance 
certificate constrained to deletions (see~\eqref{eqn:edit-dist-cert}). 
This will allow us to compare Text-CRS and \randel.

\begin{proposition} \label{prop:textcrs-del-conversion}
  The Text-CRS deletion certificate contains a deletion-based edit distance 
  certificate for any input $\vec{w} \in \reals^{n \times d}$, meaning
  $C_D(\vec{w}; \radius_R, \radius_D) \supseteq B_{\radius}(\vec{w}; \{\del\})$, if 
  \begin{equation*}
  \radius_D = \radius \text{ and } 
  \radius_R \geq \begin{cases}
    2 \radius (n - \radius), & n \geq 2 \radius, \\
    n^2/2, & n < 2 \radius.
  \end{cases}
  \end{equation*} 
\end{proposition}
\begin{proof}
  Let $\vec{w}' \in B_{\radius}(\vec{w}; \{\del\})$ be an input obtained from 
  $\vec{w}$ by deleting $l \leq \radius$ elements. 
  We observe that $\vec{w}'$ requires the greatest sum of word position 
  changes (as measured by $\| \vec{\pi} \|_1$ for the permutation matrix 
  $\vec{\pi}$) when $l$ elements are deleted at the beginning of 
  the sequence and $\vec{w}$ contains no padding words at the end. 
  In this case the permutation matrix is 
  \begin{equation*}
    \vec{\pi} = \begin{pmatrix}
      0       & I_{l} \\
      I_{n-l} & 0     \\
    \end{pmatrix}
  \end{equation*}
  with $\| \vec{\pi} \|_1 = 2 l (n - l)$.
  Taking the worst-case number of deletions $l \leq \radius$, we have
  \begin{align*}
    \radius_R &\geq \max_{l \in \{0, \ldots, \radius\}} 2 l (n - l) 
    = \begin{cases}
      2 \radius (n - \radius), & n \geq 2 \radius, \\
      n^2/2, & n < 2 \radius.
    \end{cases}
  \end{align*}
\end{proof}

\citet{zhang2024textcrs} instantiate the Text-CRS deletion certificate for 
a smoothed classifier where the smoothing mechanism permutes the embedding 
vectors uniformly at random and randomly replaces embedding vectors with 
padding with fixed probability $p$. 
For this mechanism, the largest possible value of $\radius_R$ is $n$, 
which is achieved when the classifier's confidence is 100\%. 
Combining $\radius_R \leq n$ with the inequality in 
Proposition~\ref{prop:textcrs-del-conversion} implies 
\begin{align*}
  n \geq \begin{cases}
    2 \radius (n - \radius), & n \geq 2 \radius, \\
    n^2/2, & n < 2 \radius.
  \end{cases} \Leftrightarrow
  \radius \leq \begin{cases}
    n, & n \leq 2, \\
    0, & n > 2.
  \end{cases}
\end{align*}
Hence the edit distance certificate is vacuous ($\radius = 0$) when the 
maximum sequence length $n > 2$.

\subsection{Word-level insertion}
Text-CRS covers insertion using a certificate that constrains the 
perturbation of the embedding vectors (in $\ell_2$-distance) and the sum 
of word position changes.

\begin{definition}
  A Text-CRS insertion certificate at input $\vec{w} \in \reals^{n \times d}$ 
  is a set of inputs parameterized by two radii $\radius_R, \radius_I \geq 0$:
  \begin{align*}
    & C_I(\vec{w}; \radius_R, \radius_I) = \{ \, 
      T_{\vec{\delta}, \vec{\pi}}(\vec{w}) \in \reals^{n \times d} : 
        \vec{\pi} \in \mathcal{P}_n, \\
      & \quad \vec{\delta} \in \reals^{n \times d}, \|\vec{\pi}\|_1 < \radius_R, 
      \|\vec{\delta}\|_2 < \radius_I
    \, \}.
  \end{align*}
\end{definition}

We are interested in determining values of $\radius_R$ and $\radius_I$ such 
that the Text-CRS insertion certificate covers a standard edit distance 
certificate constrained to insertions (see~\eqref{eqn:edit-dist-cert}).

\begin{proposition} \label{prop:textcrs-ins-conversion}
  Let $E$ denote the set of $d$-dimensional embedding vectors (covering all 
  possible words) and let 
  $D_\star \coloneqq \max_{\vec{e}_1, \vec{e}_2 \in E} \|\vec{e}_1 - \vec{e}_2\|_2$.  
  The Text-CRS insertion certificate contains an insertion-based edit distance 
  certificate for any input $\vec{w} \in E^n$, meaning
  $C_I(\vec{w}; \radius_R, \radius_I) \supseteq B_{\radius}(\vec{w}; \{\ins\})$, 
  if 
  \begin{align*}
    &\radius_I \geq \sqrt{\radius} D_\star
      \text{ and } \radius_R \geq  2 \radius (n - \radius), 
      & \text{when } n \geq 2 \radius, \\
    &\sqrt{n / 2} D_\star \text{ and } \radius_R \geq n^2/2, 
      & \text{when } n < 2 \radius.
  \end{align*} 
\end{proposition}
\begin{proof}
  Let $\vec{w}' \in B_{\radius}(\vec{w}; \{\ins\})$ be an input obtained from 
  $\vec{w}$ by inserting $l \leq \radius$ elements. 
  We observe that the corresponding $\| \vec{\delta} \|_2$ is maximized when 
  the $l$ inserted vectors are a distance $D_\star$ away from the 
  $l$ vectors at the end of the original input $\vec{w}$. 
  In this case $\| \delta \|_2 = \sqrt{l {D_\star}^2} = \sqrt{l} D_\star$.

  We observe that $\vec{w}'$ requires the greatest sum of word position 
  changes (as measured by $\| \vec{\pi} \|_1$ for the permutation matrix 
  $\vec{\pi}$) when the $l$ elements are inserted at the beginning of 
  the sequence. 
  In this case the permutation matrix is 
  \begin{equation*}
    \vec{\pi} = \begin{pmatrix}
      0       & I_{n-l} \\
      I_{l}   & 0     \\
    \end{pmatrix}
  \end{equation*}
  with $\| \vec{\pi} \|_1 = 2 l (n - l)$.

  Taking the worst case $\| \vec{\pi} \|_1$ with respect to the number of 
  insertions $l \leq \radius$ we have 
  \begin{align*}
    \radius_R &\geq \max_{l \in \{0, \ldots, \radius\}} 2 l (n - l) 
    = \begin{cases}
      2 \radius (n - \radius), & n \geq 2 \radius, \\
      n^2/2, & n < 2 \radius,
    \end{cases}
  \end{align*}
  where the maximizer is $l = \radius$ for the first case and $l = n / 2$ 
  for the second case. 
  Hence we have 
  \begin{equation*}
    r_I \geq \begin{cases}
      \sqrt{\radius} D_\star, & n \geq 2 \radius, \\
      \sqrt{n / 2} D_\star, & n < 2 \radius.
    \end{cases}
  \end{equation*}
\end{proof}

\citet{zhang2024textcrs} instantiate the Text-CRS insertion certificate for 
a smoothed classifier where the smoothing mechanism permutes the embedding 
vectors uniformly at random and perturbs the embedding vectors with 
Gaussian noise with fixed scale parameter $\sigma$. 
For this mechanism, the largest possible value of $\radius_R$ is $n$, 
which is achieved when the classifier's confidence is 100\%. 
Combining $\radius_R \leq n$ with the inequalities in 
Proposition~\ref{prop:textcrs-ins-conversion} implies 
\begin{align*}
  \radius \leq \begin{cases}
    \min \{n, \floor{(r_I / D_\star)^2}\} & n \leq 2, \\
    0, & n > 2.
  \end{cases}
\end{align*}
For the values of $r_I$ and $D_\star$ reported by 
\citeauthor{zhang2024textcrs}, $r_I / D_\star < 1$. 
Hence the edit distance certificate is vacuous ($\radius = 0$) for all 
sequences.